\documentclass[12pt, reqno]{amsart}

\usepackage[margin=0.75in]{geometry}
\usepackage[
  bookmarks=true]{hyperref}
\usepackage[OT2, T1]{fontenc}
\usepackage[english]{babel}
\usepackage[utf8]{inputenc}
\usepackage{csquotes}
\usepackage[final]{microtype}
\usepackage{lmodern}
\usepackage{amsthm}
\usepackage{amssymb}
\usepackage{mathrsfs}
\usepackage{enumerate}
\usepackage{tikz-cd} 
\usepackage{tikz}
\usepackage{multicol}
\usepackage{multirow}
\usepackage{tikz-qtree}
\usepackage{rotating}
\usepackage{graphicx}
\usepackage{comment}
\usetikzlibrary{arrows,calc,matrix,trees,arrows.meta,positioning,decorations.pathreplacing,bending}

\newtheorem{theorem}{Theorem}[section]

\newtheorem{proposition}[theorem]{Proposition}
\newtheorem{lemma}[theorem]{Lemma}
\newtheorem{corollary}[theorem]{Corollary}

\theoremstyle{definition}

\newtheorem{remark}[theorem]{Remark}
\newtheorem{definition}[theorem]{Definition}
\newtheorem{example}[theorem]{Example}

\newcommand{\Z}{\mathbb{Z}}

\newcommand{\R}{\mathbb{R}}

\newcommand{\Top}{\text{Top}} 
\newcommand{\Ab}{\text{Ab}} 
\begin{document}
\title{A Functorial Formulation of Neighborhood Aggregating Deep Learning}

\author{Sun Woo Park}
\address{Max Planck Institute for Mathematics,  Vivatsgasse 7, 53111 Bonn, Germany}
\email{s.park@mpim-bonn.mpg.de}

\author{Yun Young Choi}
\address{SolverX, Gangseo-gu, Seoul 07801, Republic of Korea}
\email{young@solverx.ai}

\author{U Jin Choi}
\address{Korea Advanced Institute of Science and Technology, Department of Mathematical Sciences, 291 Daehak-ro, Yuseong-gu, Daejeon 34141, Republic of Korea}
\email{ujinchoi@kaist.ac.kr}

\author{Youngho Woo}
\address{National Institute for Mathematical Sciences, 463-1 Jeonmin-dong, Yuseong-gu, Daejeon, 34047, Republic of Korea}
\email{youngw@nims.re.kr}

\begin{abstract}
    We provide a mathematical interpretation of convolutional (or message passing) neural networks by using presheaves and copresheaves of the set of continuous functions over a topological space. Based on this interpretation, we formulate a theoretical heuristic which elaborates a number of empirical limitations of these neural networks by using obstructions on such sets of continuous functions over a topological space to be sheaves or copresheaves.
\end{abstract}

\maketitle
\tableofcontents

\section{Introduction}
The universal approximation theorem proves the effectiveness of neural networks as operators for approximating continuous functions of real numbers.

\begin{theorem}[Universal Approximation Theorem]
Given a compact subset $K \subset \mathbb{R}^n$, let $f: K \to \mathbb{R}^m$ be a continuous function. Let $\sigma: \mathbb{R} \to \mathbb{R}$ be a predetermined continuous function. Then the following two conditions are equivalent.
\begin{enumerate}
\item $\sigma$ is not a polynomial function.
\item For every $\epsilon > 0$, there exist positive numbers $k > 0$, matrices $A \in \mathbb{R}^{k \times n}$, $B \in \mathbb{R}^{k}$, and $C \in \mathbb{R}^{m \times k}$ such that
\begin{equation}
    \sup_{x \in K} \| f(x) - C \cdot \left( \sigma^{[k]} \left( A \cdot x + B \right) \right) \| < \epsilon.
\end{equation}
where $\cdot$ is the matrix multiplication operator, and $\sigma^{[k]}: \mathbb{R}^k \to \mathbb{R}^k$ is a function whose coordinate-wise functions are all equal to $\sigma$.
\end{enumerate}
\end{theorem}

The above formulation, see for example \cite{Cybenko89, Fu89, HSW89}, shows that a neural network comprised of two fully connected layers of arbitrary depth generates a dense subset of the set of continuous functions from $\mathbb{R}^n$ to $\mathbb{R}^m$ under the compact open topology. Other formulations include the universal approximation theorem for fully connected neural networks with arbitrary number of layers and architectural variations of neural networks \cite{KT20,LPW17,PYLS21}. These formulations of universal approximation theorems rigorously model asymptotic properties of neural networks in approximating compactly supported continuous functions over the real numbers. 

However, the universal approximation theorem shows limited capability in analyzing empirical properties of variants of neural networks with finite layers of finite depths. Some of these empirical limitations and additional measures to improve such limitations, in particular for convolutional (or message passing) neural networks that analyze image data sets, graph data sets, or time series data, can be listed as follows.
\begin{enumerate}
    \item \textbf{Non-unique Gluing}: Convolutional (or message passing) neural networks with max pooling layers are vulnerable from recognizing a family of images comprised of juxtapositions of isomorphic local components as identical objects.
    \item \textbf{Adversarial Attacks}: Perturbation in input data sets often mislead convolutional (or message passing) neural networks to misidentify image or graph data sets.
    \item \textbf{Dataset Dependency}: Suitable architectural choices in constructing neural networks have to be chosen to produce state-of-the-art performance in analyzing innate properties of each data set.
    \item \textbf{Topological Inferences}: Topological data analysis techniques often enhances performances of convolutional (or message passing) neural networks in analyzing image or graph data sets.
\end{enumerate}

This paper intends to provide a mathematical framework which gives a theoretical heuristic on the common origins of these apparently disjoint empirical limitations of convolutional (or message passing) neural networks. The key overarching thesis of the paper can be summarized into two points as follows.

\begin{theorem}[Simplification of Theorems \ref{theorem:first_limitation}, \ref{theorem:second_limitation} \ref{theorem:third_limitation}, \ref{theorem:fourth_limitation}]
Let $X$ be a locally compact connected Hausdorff topological space. We denote by $DL^m$ a convolutional (or message passing) neural network with $m$ layers (or a discrete deep learning algorithm with $m$ layers satisfying neighborhood aggregating, see a combination of Definitions \ref{def:ml} and \ref{def:neighborhood_aggregate}).
\begin{enumerate}
    \item \textbf{Functorial Interpretation} Any $DL^m$ that accepts data set defined over $X$ as input approximates a global section of a presheaf and a copresheaf of the set of continuous functions over the space $X$.
    \item \textbf{Empirical Properties} The four aforementioned empirical properties of $DL^m$ originate from functorial properties of such presheaves and copresheaves.
\end{enumerate}
\end{theorem}

To identify convolutional (or message passing) neural network as a functor from the category of open subsets of $X$ to the category of real vector spaces, we define what is called the presheaf of dual cosheaves and copresheaf of dual sheaves. The local sections of these functors are continuous functions from $\mathbb{R}^n$ to $\mathbb{R}^m$ for some positive integers $n,m > 0$. The global sections correspond to continuous functions the convolutional (or message passing) neural network aims to approximate. Hence, we demonstrate that the functorial formulation of these neural networks provides a novel mathematical framework for deducing theoretical support for empirically verified architectural limitations such neural networks may possess, and formulating statements on the classes of continuous functions that can be approximated by convolutional (or message passing) neural networks with finitely many layers of finite depth.

\subsection{Related Studies}

There have been several previous studies which focus on utilizing cellular sheaves and sheaf theory to give a theoretical analysis on limitations of message passing neural networks. These include oversmoothing, analyzing heterophilic graphs, and effectively extracting global properties of the underlying geometric space of data sets \cite{CPWC22}. Hanson and Gebhart applied cellular sheaves to construct sheaf neural network, which generalizes diffusion operators underlying graph neural networks \cite{HG20}. The motivation to utilize sheaf theory to enhance graph neural networks or other forms of neural networks was further explored in subsequent works, such as the work by Bodnar et al. \cite{BG22} (sheaf convolutional network), Barbero et al. \cite{BB22, BB22b} (which incorporated connection Laplacians and attention mechanisms), He et al. \cite{HB24} (positional encoding), Braithwaite et al. (heterogeneous sheaf neural networks), Hajij et al. (copresheaf topological neural networks) \cite{HB25}, and Borgio et al. \cite{BS25} (polynomial neural sheaf diffusion).

The most relevant and recent groundbreaking work in applying sheaf theory and copresheaf theory to enhancing neural networks is the work by Hajij et al. on Copresheaf Topological Neural Networks (CTNN) \cite{HB25}. This is the first published work which uses the theory of copresheaves to provide a provide an overarching theoretical formulation of various types of deep learning architectures, such as convolutional neural networks, transformers, and message passing neural networks. The motivation for devising CTNN originates from comparisons between cellular sheaves and copresheaves on combinatorial complexes. They demonstrate that CTNN enhances performances in addressing many challenges in representation learning in comparison to conventional deep learning algorithms, such as preventing oversmoothing, analyzing heterophilic graphs, and analyzing non-Euclidean datasets.

\subsection{Novelty and limitations}
Unlike the previous studies which build upon the notion of cellular sheaves, our paper provides a theoretical formulation of convolutional (or message passing) neural networks by using the following three new perspectives.
\begin{itemize}
    \item \textbf{Skyscraper sheaves/cosheaves}: We use presheaves (or sheaves) and copresheaves (or cosheaves) to analyze deep learning techniques. To do so, we use presheaf / copresheaf of continuous functions induced from skyscraper cosheaves / sheaves to interpret deep learning architectures as approximators of their local sections. The key insight we use is the fact that the skyscraper sheaf is also a cosheaf. This overlapping duality allows us to assess the capabilities of convolutional (or message passing) neural networks in constructing vector representations by gluing locally defined vector representations.

    \item \textbf{Obstructions}: By using the language of presheaves and copresheaves, we can use our assessment on capabilities of convolutional (or message passing) neural networks to obtain theoretical limitations of these architectures. We formulate these limitations by considering obstructions for a presheaf to be a sheaf and a copresheaf to be a cosheaf. To elaborate, there are parts of the sheaf axioms (or cosheaf axioms) that convolutional (or message passing) neural networks violate. Such violations can be used to pinpoint limitations of these networks, such as non-unique gluing of local vector representations, adversarial attacks, and performance dependency on datasets. We also demonstrate trivial presheaf cohomology (and copresheaf homology) for these presheaves and copresheaves, thereby suggesting the merits of incorporating topological inferences when needed.

    \item \textbf{Examples and future directions}: We reformulate previously studied deep learning architectures using presheaf / copresheaf of continuous functions induced from skyscraper cosheaves / sheaves. Some examples include convolutional neural networks, message passing neural networks, recurrent neural networks, and attention-transformers. We also propose that other types of sheaves or cosheaves other than the presheaf / copresheaf of continuous functions may give rise to a zoo of novel deep learning algorithms that can surpass previously studied architectures.
\end{itemize}
As many theoretical analyses do, our theoretical framework does not encompass all deep learning techniques, and possess some limitations. Our framework, for example, does not address strengths or limitations that novel deep learning algorithms built upon sheaves or cosheaves other than presheaf / copresheaf of continuous functions may possess. Our framework also does not explain strengths or limitations on concurrently using multiples of deep learning algorithms together, such as multi-agent systems. Nevertheless, we hope that future research may focus on exploring whether sheaf or cosheaf theory can be utilized to effectively analyze different classes of deep learning techniques or different ways to collectively utilize a family of them.

\subsection{Organization}
We organize the paper in the following manner. Section \ref{section:sheaves_cosheaves} focuses on constructing a presheaf (or a copresheaf) of continuous functions induced from cosheaves (or sheaves), whose functorial properties we will analyze in Proposition \ref{prop:sheaf_fail}. Section \ref{section:deep_learning} utilizes the mathematical framework from Section \ref{section:sheaves_cosheaves} to define discrete deep learning algorithms with neighborhood aggregating layers, and provide a theoretical argument for why certain empirical, as will be shown in Theorems \ref{theorem:first_limitation}, \ref{theorem:second_limitation} \ref{theorem:third_limitation}, and \ref{theorem:fourth_limitation}. Section \ref{section:deep_learning_examples} discusses how certain variants of convolutional (or message passing) neural networks, message passing neural networks for graphs, and recurrent neural networks can be reformulated using the mathematical model provided in Section \ref{section:deep_learning}. We finish the manuscript with Section \ref{section:deep_learning_future}, where we briefly discuss some deep learning algorithms which overcome the aforementioned empirical drawbacks, such as attention-transformers, persistent homological techniques, and neural ODEs. We also discuss correspondences between deep learning algorithms that process dynamic time series data defined over graphs and those which process data defined over 2-dimensional spaces such as image data sets.


\subsection*{Acknowledgements}
The majority of the work was completed while the first and the second author were members of the National Institute for Mathematical Sciences (NIMS) up until August of 2022. Sun Woo Park, Yun Young Choi, and Youngho Woo were supported by the National Institute for Mathematical Sciences (NIMS) grant funded by the Korean Government (MSIT) B22920000. The first author would like to thank Max Planck Institute for Mathematics for providing its hospitality, during his stay in which some updates in the manuscript had been made. We would like to thank Asung Kil for constructive comments and suggestions.

\section{Sheaves and Cosheaves} \label{section:sheaves_cosheaves}

\subsection{Preliminary}
In this section, we give a brief review of sheaf and cosheaf theory, which are algebraic tools effective for representing topological properties of a space $X$ as real (or complex) vectors. This subsection closely follows Chapter 2 of \cite{Ha77}, Chapters 5 and 6 of \cite{Br97}, and Chapter 7 of \cite{Cu13}. Throughout this section, we denote by $X$ a locally compact connected Hausdorff topological space.

\begin{definition}[Presheaf]
A presheaf $\mathcal{F}$ of abelian groups over $X$ is a contravariant functor from the category of open subsets of a topological space $X$ to the category of abelian groups
\begin{equation}
    \mathcal{F}: \Top(X)^{Op} \to \Ab
\end{equation}
\end{definition}

\begin{definition}[Sheaf]
A sheaf $\mathcal{F}$ of abelian groups over $X$ is a presheaf of abelian groups over $X$ which satisfies the exact sequence 
\begin{equation} \label{eq:sheaf_axiom}
    0 \to \mathcal{F}(U) \to \prod_{\alpha} \mathcal{F}(U_\alpha) \to \prod_{\alpha, \beta} \mathcal{F}(U_\alpha \cap U_\beta)
\end{equation}
for all collections of open sets $\{U_\alpha\}$ with $U = \cup_{\alpha} U_\alpha$.
\end{definition}

\begin{definition}[copresheaf]
A copresheaf $\widetilde{\mathcal{F}}$ of abelian groups over $X$ is a covariant functor from the category of open subsets of a topological space $X$ to the category of abelian groups
\begin{equation}
    \widetilde{\mathcal{F}}: \Top(X) \to \Ab
\end{equation}
\end{definition}

\begin{definition}[Cosheaf]
A cosheaf $\widetilde{\mathcal{F}}$ of abelian groups over $X$ is a copresheaf of abelian groups over $X$ which satisfies the exact sequence
\begin{equation} \label{eq:cosheaf_axiom}
    \bigoplus_{\alpha, \beta} \widetilde{\mathcal{F}}(U_\alpha \cap U_\beta) \to \bigoplus_{\alpha} \widetilde{\mathcal{F}}(U_\alpha) \to \widetilde{\mathcal{F}}(U) \to 0
\end{equation}
for all collections of open sets $\{U_\alpha\}$ with $U = \cup_{\alpha} U_\alpha$.
\end{definition}

\begin{example}[Constant Sheaf / Cosheaf]
Fix an abelian group $A$. The constant sheaf/cosheaf $\hat{A}$ over $X$ is given by
\begin{equation}
    \hat{A}(U) = A
\end{equation}
for any open neighborhood $U$.
\end{example}

\begin{example}[Skyscraper Sheaf / Cosheaf]
Let $x \in X$ be a point. Fix an abelian group $A$. The skyscraper sheaf/cosheaf at $x$, denoted as $S_x$, is given by
\begin{equation}
    S_x(U) := \begin{cases}
    A & \text{ if } x \in U \\
    0 & \text{ otherwise}
    \end{cases}
\end{equation}
\end{example}
We note here that the function $S_x$ is both a sheaf and a cosheaf (see Section 6.27 of \cite{St21} and Definition 3.3.3 \cite{Cu13}).

\begin{example}
Let $X$ be a locally compact connected cell complex. The presheaf of continuous real valued functions $C^0_{X}$ defined as
\begin{equation}
    C^0_X(U) := \{f : U \to \mathbb{R}^l \; | \; f \text{ continuous} \}
\end{equation}
is a sheaf over $X$. The presheaf of $k$-differentiable functions $C^k_X$ is also a sheaf over $X$.
\end{example}

\begin{example}
Let $X$ be a locally compact connected cell complex. The copresheaf of compactly supported functions $\Omega^0_X$ defined as
\begin{equation}
    \Omega^0_X(U) := \{f : U \to \mathbb{R} \; | \; \text{supp}(f) \text{ compact } \}
\end{equation}
is a cosheaf over $X$. The copresheaf of differential $k$-forms $\Omega^k_X$ is also a cosheaf over $X$. One can use a partition of unity of a topological space $X$ to prove that the copresheaf $\Omega^k_X$ satisfies the cosheaf axiom (\ref{eq:cosheaf_axiom}).
\end{example}

\subsection{Presheaf of dual cosheaves and copresheaf of dual sheaves}
One of the central objects we will discuss in this manuscript is the presheaf of dual cosheaves and the copresheaf of dual sheaves.
\begin{definition}[Presheaf of Dual Cosheaves]
Let $\widetilde{\mathcal{F}}$ be a cosheaf of real vector spaces over $X$, and let $V$ be a fixed finite dimensional real vector space.
\begin{enumerate}
\item The presheaf of linear morphisms induced from $\widetilde{\mathcal{F}}$, denoted as $\text{Hom}(\widetilde{\mathcal{F}},V)$ is given by
\begin{equation}
    \text{Hom}(\widetilde{\mathcal{F}},V)(U) := Hom(\widetilde{\mathcal{F}}(U),V)
\end{equation}
\item The presheaf of continuous functions induced from $\widetilde{\mathcal{F}}$, denoted as $\mathcal{C}^0(\widetilde{\mathcal{F}},V)$ is given by
\begin{equation}
    \mathcal{C}^0(\widetilde{\mathcal{F}},V) := \{f : \widetilde{\mathcal{F}}(U) \to V \; | \; f \text{ continuous } \}
\end{equation}
\end{enumerate}
\end{definition}

\begin{definition}[copresheaf of Dual Sheaves]
Let $\mathcal{F}$ be a sheaf of real vector spaces over $X$, and let $V$ be a fixed finite dimensional real vector space.
\begin{enumerate}
\item The copresheaf of linear morphisms induced from $\mathcal{F}$, denoted as $\text{Hom}(\mathcal{F},V)$ is given by
\begin{equation}
    \text{Hom}(\mathcal{F},V)(U) := Hom(\mathcal{F}(U),V)
\end{equation}
\item The copresheaf of continuous functions induced from $\mathcal{F}$, denoted as $\mathcal{C}^0(\mathcal{F},V)$ is given by
\begin{equation}
    \mathcal{C}^0(\mathcal{F},V) := \{f : \mathcal{F}(U) \to V \; | \; f \text{ continuous } \}
\end{equation}
\end{enumerate}
\end{definition}

We immediately obtain that the presheaf $\text{Hom}(\widetilde{\mathcal{F}},\R^k)$ induced from the cosheaf $\widetilde{\mathcal{F}}$ is a sheaf, and the copresheaf $\text{Hom}(\mathcal{F},\R^k)$ is a cosheaf.
\begin{proposition}[\cite{Br97}, Proposition 5.1.10.]
Given a cosheaf $\widetilde{\mathcal{F}}$ and $k$ a fixed positive number, the presheaf $\text{Hom}(\widetilde{\mathcal{F}},\R^k)$ is a sheaf. Likewise, given a sheaf $\widetilde{\mathcal{F}}$, the copresheaf $\text{Hom}(\mathcal{F},\R^k)$ is a cosheaf.
\end{proposition}
\begin{proof}
The proposition follows immediately from the fact that $\text{Hom}(-,\R^k)$ is exact. Note that $\R^k$ is an injective $\R$-module, which makes the left-exact functor $\text{Hom}(-,\R^k)$ exact.
\end{proof}

In contrast, the presheaf ${C}^0(\mathcal{F},V)$ and copresheaf ${C}^0(\widetilde{\mathcal{F}},V)$ do not necessarily satisfy the sheaf (or cosheaf) axioms. We suspect that additional structural properties of the sheaf $\mathcal{F}$ (or the cosheaf $\widetilde{\mathcal{F}}$) may be required to ensure that these functors satisfy the respective axioms. In the following key proposition, we show that if $\mathcal{F}$ is a finite direct sum of skyscraper sheaves (or cosheaves), then the functors ${C}^0(\mathcal{F},V)$ do not satisfy the sheaf (or cosheaf) axioms.

Before we discuss further about the sheaf-theoretic properties of $C^0(\mathcal{F},V)$, let us recall that the procedure of constructing a data set defined over a topological space $X$ consists of taking measurements, possibly with some perturbations due to sensor noises, over finitely many points $\{x_i\}_i=1^N$ of $X$. These procedures can be described using the definition of pushforward of constant sheaves over a discrete set to the topological space.

\begin{definition}
Given a continuous map of topological spaces $f:X \to Y$, the pushforward sheaf $f_*(\mathcal{F})$ of a sheaf $\mathcal{F}$ over $X$ is a sheaf over $Y$ given by
\begin{equation}
    f_*(\mathcal{F})(V) := \mathcal{F}(f^{-1}(V)).
\end{equation}
\end{definition}

\begin{definition}
Given an open continuous map of topological spaces $f:X \to Y$, the pullback sheaf $f^{-1}(\mathcal{G})$ of a sheaf $\mathcal{G}$ over $Y$ is a sheaf over $X$ given by
\begin{equation}
    f^{-1}(\mathcal{G})(U) = \mathcal{G}(f(U)).
\end{equation}
\end{definition}

\begin{definition}
Let $A$ be a finite set of points, endowed with discrete topology. Denote by $C_{A,l}$ the constant sheaf (and cosheaf) of $l$-dimensional vectors. The presheaf (and the copresheaf) of $k$-dimensional convolutional (or message passing) neural networks, denoted as $C_0(i_{A,l},\mathbb{R}^k)$, is constructed as follows.
\begin{enumerate}
    \item Let $i_{A,l}$ be the pushforward sheaf of $C_{A,l}$ with respect to the inclusion map $i:A \to X$. If necessary, one can take the pullback of the pushforward sheaf of $C_{A,l}$ with respect to the inclusion map $i:A \to X$ and the universal covering map $\pi:\widetilde{X} \to X$.
    \item Let $Hom(i_{A,l},\mathbb{R}^k)$ be the presheaf of linear morphisms (or the copresheaf of linear morphisms) induced from the sheaf $i_{A,l}$.
    \item Let $C_0(i_{A,l},\mathbb{R}^k)$ be the presheaf of continuous functions (or the copresheaf of continuous functions) induced from the sheaf $i_{A,l}$.
\end{enumerate}
We refer to Figures \ref{fig:pushforward} and \ref{fig:induced_cont} for a visual demonstration of presheaves and copresheaves constructed in the definition.
\end{definition}

\begin{figure}
    \centering
    \includegraphics[width=120mm]{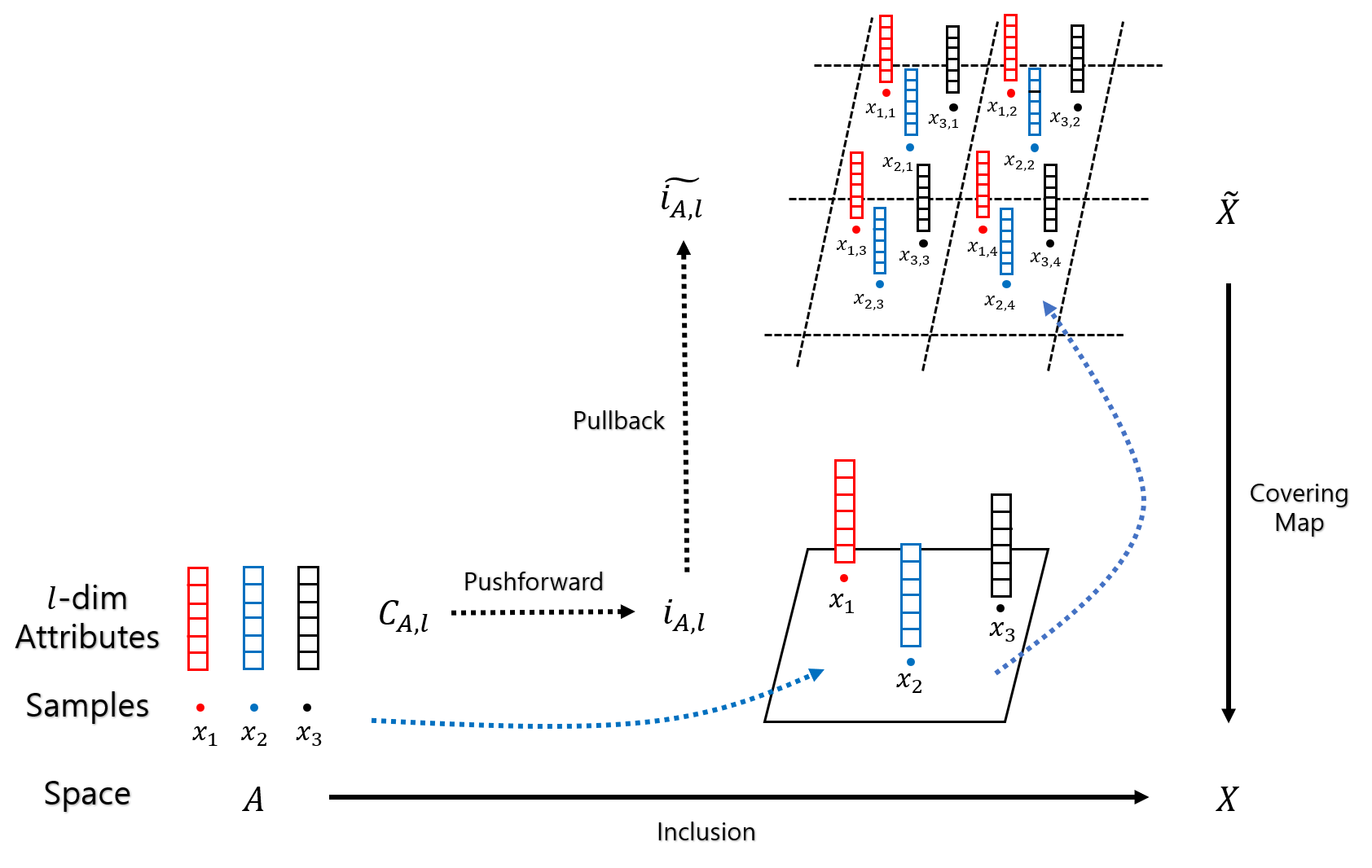}
    \caption{Construction of a pushforward of constant sheaves / cosheaves over discrete sets}
    \label{fig:pushforward}
\end{figure}    
    
\begin{figure}
    \centering
    \includegraphics[width=120mm]{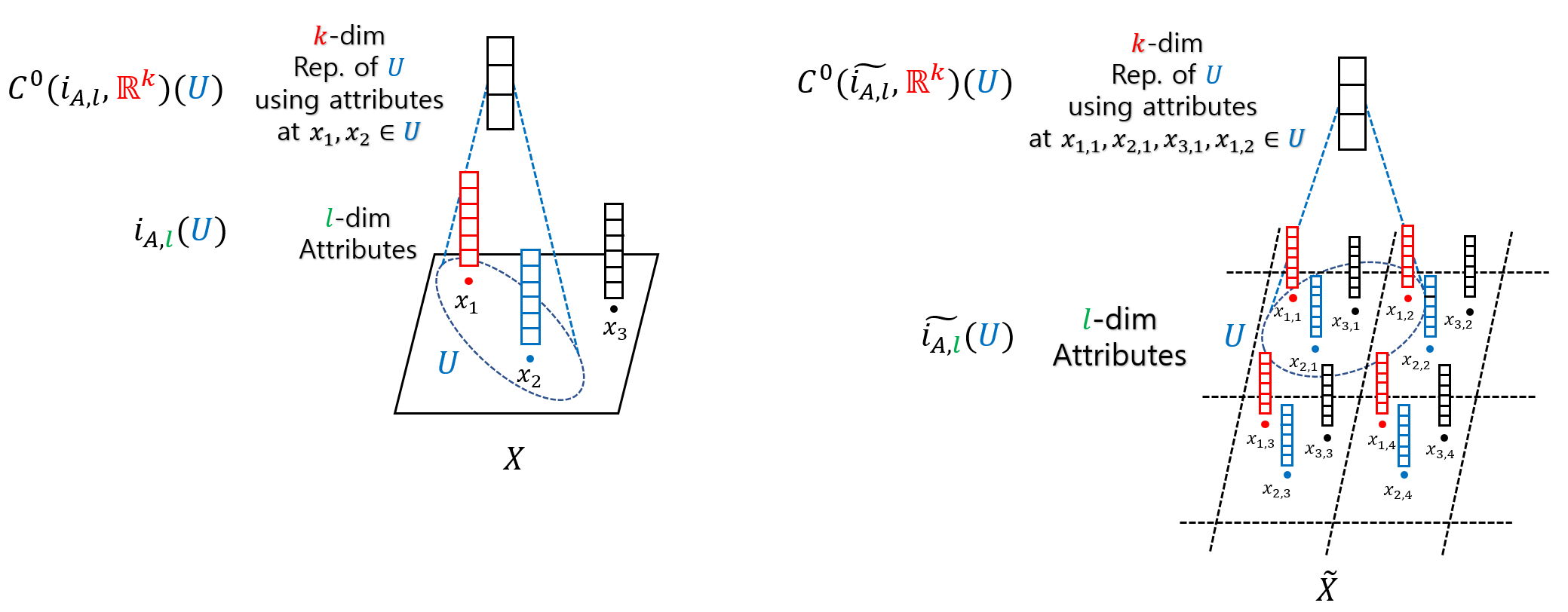}
    \caption{Construction of a functor describing the presheaf / copresheaf of continuous functions induced from the sheaf $i_{A,l}$ over a topological space $X$ and its universal cover $\widetilde{X}$.}
    \label{fig:induced_cont}
\end{figure}

A crucial property we will explore in this section is that the presheaf (and the copresheaf) of continuous functions induced from the pushforward sheaf (and the pushforward cosheaf) is not a sheaf (and a cosheaf). This property will play a key role in constructing a mathematical framework for understanding the empirical limitations of certain classes of deep learning algorithms in the upcoming sections.

\begin{proposition} \label{prop:sheaf_fail}
Let $X$ be a locally compact Hausdorff topological space. Fix a finite set of points $\{x_i\}_{i=1}^N$ in $X$. Let $S_{x_i}$ be the skyscraper sheaf (or cosheaf $S_{x_i}^{op}$) of real vector space at the point $x_i$, given by
\begin{equation}
    S_{x_i}(U) := \begin{cases}
    \R^{l_i} \text{ if } x_i \in U \\
    0 \text{ otherwise }
    \end{cases}
\end{equation}
\begin{enumerate}
    \item The copresheaf $C^0(i_{A,l},\R^k)$ is not a cosheaf.
    \item The presheaf $C^0(i_{A,l}^{op},\R^k)$ is not a sheaf. 
\end{enumerate}
\end{proposition}
\begin{proof}

\medskip
\textbf{The copresheaf $C^0(i_{A,l},\R^k)$}
\medskip

Let us regard the functor $i_{A,l}$ as the sheaf of real vector spaces over $X$. The inclusion of open subsets $U \to V$ induces a projection map of real vector spaces.
\begin{align}
\begin{split}
    res_{V,U}: i_{A,l}(V) &\to i_{A,l}(U) \\
    (y_1,y_2, \cdots, y_{\sum_{x_i \in V} l_i}) &\mapsto (y_1,y_2, \cdots, y_{\sum_{x_i \in U} l_i})
\end{split}
\end{align}

As for the copresheaf $C^0(i_{A,l},\R^k)$, the inclusion $U \to V$ induces a morphism of functions
\begin{align}
    i_{U,V}: C^0(i_{A,l},\R^k)(U) &\to C^0(i_{A,l}, \R^k)(V) \\
    f &\mapsto f \circ res_{V,U}
\end{align}

We refer to Figures \ref{fig:res} and \ref{fig:ind} for illustrations on how the morphisms $res_{V,U}$ and $i_{U,V}$ given two open sets $U \subset V$ are constructed. 
\begin{figure}
        \centering
        \includegraphics[width=100mm]{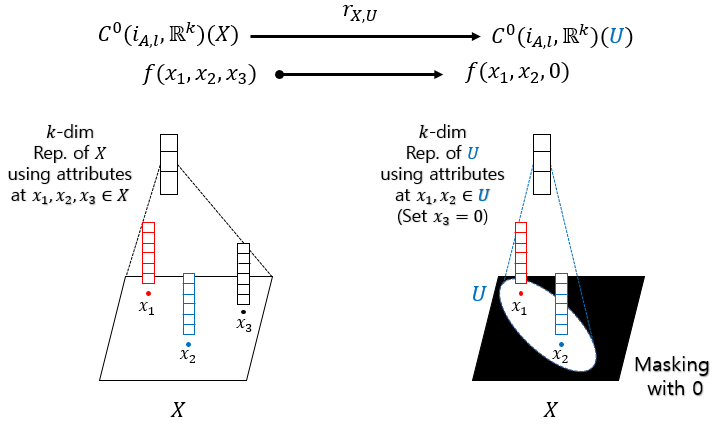}
        \caption{Definition of the restriction map $res_{X,U}$}
        \label{fig:res}
\end{figure}
\begin{figure}
        \centering
        \includegraphics[width=100mm]{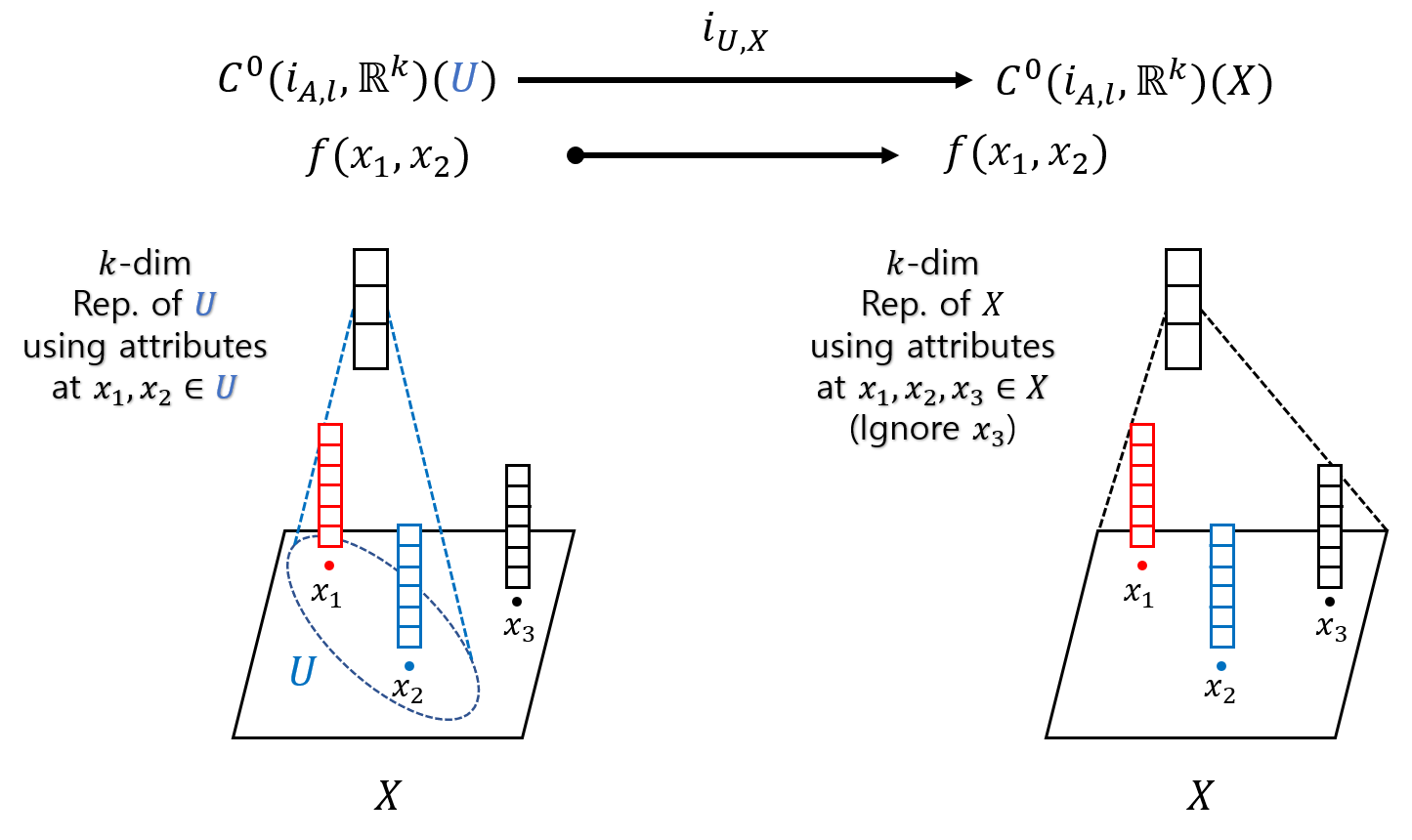}
        \caption{Definition of the inclusion map $i_{U,X}$}
        \label{fig:ind}
\end{figure}

Let $U \subset X$ be any open neighborhood. Without loss of generality, suppose the open set $U$ contains all the points $\{x_1,x_2,\cdots,x_N\}$. Let $\{U_\alpha\}_\alpha$ be an open cover of $U$ such that there exists a unique point $x_\alpha \in \{x_i\}_{i=1}^N$ such that $x_\alpha \in U_\alpha$. Then the morphism
\begin{equation} \label{equation:copresheaf_map}
    \bigoplus_{\alpha} C^0(i_{A,l},\R^k)(U_\alpha) \to C^0(i_{A,l},\R^k)(U)
\end{equation}
is not surjective. Consider the function 
\begin{equation}
    f(y_1, \cdots, y_{\sum_{i=1}^N l_i}) = \left( \prod_{j=1}^{\sum_{i=1}^N l_i} y_j, \cdots, \prod_{j=1}^{\sum_{i=1}^N l_i} y_j \right)
\end{equation}
Then there doesn't exist a collection of functions $\{g_i: \R^{l_i} \to \R^k \}_{i=1}^N$ such that $f = \sum_\alpha g_\alpha \circ res_{U_\alpha,U}$. Indeed, for each $i$-th coordinate of $f$, $g_i \neq \prod_{j=1}^{\sum_{i=1}^N l_i} y_j$.

\medskip
\textbf{The presheaf $C^0(i_{A,l},\R^k)$}
\medskip

If we consider $S_{x_i}$ as the skyscraper cosheaf of real vector space at the point $x_i$, the presheaf $C^0(i_{A,l}, \R^k)$ is not a sheaf. The inclusion of open subsets $U \to V$ induces an injective linear map of vector spaces
\begin{align}
    \begin{split}
        i_{U,V}: i_{A,l}(U) &\to i_{A,l}(V) \\
        (y_1, y_2, \cdots, y_{\sum_{x_i \in U} l_i}) &\mapsto (y_1, y_2, \cdots, y_{\sum_{x_i \in U} l_i}, 0, \cdots, 0)
    \end{split}
\end{align}
where the image $i_{U,V}(\mathbb{R}^{\sum_{x_i \in U} l_i})$ is equal to the subspace $\mathbb{R}^{\sum_{x_i \in U} l_i} \times \{0\}$ of $\mathbb{R}^{\sum_{x_i \in V} l_i}$.

As for the presheaf $C^0(i_{A,l}, \mathbb{R}^k)$, the inclusion $U \to V$ induces a morphism of functions
\begin{align}
    \begin{split}
        res_{V,U}: C^0(i_{A,l}, \mathbb{R}^k)(V) &\to C^0(i_{A,l}, \mathbb{R}^k)(U) \\
        \{f: \R^{\sum_{x_i \in V} l_i} \to \R^k \} &\mapsto \{f \circ i_{U,V}: \R^{\sum_{x_i \in U} l_i} \to \R^{\sum_{x_i \in V} l_i} \to \R^k \}
    \end{split}
\end{align}

Let $U \subset X$ be any open neighborhood.  Without loss of generality, suppose the open set $U$ contains all the points $\{x_1, x_2, \cdots, x_N\}$. Let $\{U_\alpha\}_\alpha$ be an open cover of $U$ such that there exists a point $x_\alpha \in \{x_i\}_{i=1}^N$ such that $x_\alpha \not\in U_\alpha$. Then the morphism
\begin{align} \label{equation:presheaf_map}
\begin{split}
    C^0(i_{A,l},\R^k)(U) &\to \prod_{\alpha} C^0(i_{A,l}, \R^k)(U_\alpha) \\
    f &\mapsto (f \circ i_{U_\alpha,U})_{\alpha}
\end{split}
\end{align}
is not injective. We can use the same function $f: \R^{\sum_{i=1}^N l_i} \to \R^k$ defined as
\begin{equation}
    f(y_1, \cdots, y_{\sum_{i=1}^N l_i}) = \left( \prod_{j=1}^{\sum_{i=1}^N l_i} y_j, \cdots, \prod_{j=1}^{\sum_{i=1}^N l_i} y_j \right)
\end{equation}
to show that it is a non-zero function whereas for any $\alpha$, $f \circ i_{U_\alpha,U} = 0$.
\end{proof}

\begin{remark} \label{remark:sheaf_fail}
Proposition \ref{prop:sheaf_fail} shows that for any collections of open sets $\{U_\alpha\}_\alpha$ with $U = \cup_{\alpha} U_\alpha$, the presheaf (or the copresheaf) $C^0(i_{A,l}, \R^k)$ satisfies the following two relations:
\begin{align} \label{eqn:fail_condition}
    \begin{split}
        \text{Ker} \left( C^0(i_{A,l}, \R^k)(U) \to \prod_{\alpha} C^0(i_{A,l}, \R^k)(U_\alpha) \right) & \neq 0 \text{ (Locality) } \\
        \text{Im} \left( \oplus_{\alpha} C^0(i_{A,l}, \R^k)(U_\alpha) \to C^0(i_{A,l}, \R^k)(U) \right) &\neq C^0(i_{A,l}, \R^k)(U) \text{ (Surjectivity) }
    \end{split}
\end{align}
However, for any finite collections of open sets $\{U_\alpha\}_{\alpha=1}^n$ with $U = \cup_{\alpha=1}^n U_\alpha$, one can prove that $C^0(i_{A,l}, \R^k)$ satisfies the following two ``gluing'' conditions. We note that the first condition corresponds to the presheaf $C^0(i_{A,l},\R^k)$, whereas the second condition correpsonds to the cosheaf $C^0(i_{A,l},\R^k)$.
\begin{align} \label{eqn:succeed_condition}
    \begin{split}
        \text{Im} \biggl( C^0(i_{A,l}, \R^k)(U) &\to \prod_{\alpha} C^0(i_{A,l}, \R^k)(U_{\alpha}) \biggr) \\
        &= \text{Ker} \biggl( \prod_{\alpha} C^0(i_{A,l}, \R^k)(U_\alpha) \to \prod_{\alpha, \beta} C^0(i_{A,l}, \R^k)(U_\alpha \cap U_\beta) \biggr), \\
        \text{Ker} \biggl( \oplus_{\alpha} C^0(i_{A,l}, \R^k)(U_{\alpha}) &\to  C^0(i_{A,l}, \R^k)(U) \biggr) \\
        &= \text{Im} \biggl( \oplus_{\alpha, \beta} C^0(i_{A,l}, \R^k)(U_\alpha \cap U_\beta) \to \oplus_{\alpha} C^0(i_{A,l}, \R^k)(U_\alpha) \biggr).
    \end{split}
\end{align}
We give a sketch of the proof of the claim above for the case $\{U_\alpha\}_{\alpha=1}^3$, the proof for the general case of which we omit in this manuscript. By definition, one can easily check that the left hand sides of the equations above are contained in the right hand sides of the equations. To prove the converse of the first equation above, we denote by $f_i \in C^0(i_{A,l},\R^k)(U_i)$ the local sections defined over $U_i$'s. Denote by $f_{i,j} := res_{U_i,U_i\cap U_j}(f_i)$. Suppose that
\begin{equation*}
    (f_1,f_2,f_3) \in \text{Ker} \left( \prod_{\alpha=1}^3 C^0(i_{A,l}, \R^k)(U_\alpha) \to \prod_{1 \leq \alpha < \beta \leq 3} C^0(i_{A,l}, \R^k)(U_\alpha \cap U_\beta) \right)
\end{equation*}
Then by definition, $f_{i,j} = f_{j,i}$ for all $1 \leq i < j \leq 3$. We can hence define the local section $f_{i,j,k} := res_{U_i \cap U_j, U_i \cap U_j \cap U_k}(f_{i,j})$ defined over the intersection of three open sets $U_i \cap U_j \cap U_k$. Then $f_{1,2,3} = f_{1,3,2} = f_{2,3,1}$. Define a function $f \in C^0(i_{A,l},\R^k)(U)$ by
\begin{equation}
    f := f_1 + f_2 + f_3 - f_{1,2} - f_{1,3} - f_{2,3} + f_{1,2,3},
\end{equation}
where we extend functions $f_i, f_{i,j}, f_{i,j,k}$ to $U$ via extension by zero. Then we obtain that $f$ restricts to local sections $(f_1,f_2,f_3)$ in $\prod_{\alpha=1}^3 C^0(i_{A,l},\R^k)(U_\alpha)$. 

The procedure of the proof for verifying the statement for copresheaf $C^0(i_{A,l},\R^k)$ is analogous. We also present the exemplary proof for the case $\{U_\alpha\}_{\alpha=1}^3$. By definition, the left hand side of the cosheaf condition contains the right hand side of the cosheaf condition. We check the converse statement. For each $\alpha = 1, 2, 3$, let $f_\alpha: \mathbb{R}^{l_\alpha} \to \mathbb{R}^k \in C^0(i_{A,l}, \R^k)(U_{\alpha})$ for each $\alpha = 1, 2, 3$, and denote by $y_1^{\alpha}, y_2^{\alpha}, \cdots, y_{l_\alpha}^{\alpha}$ the coordinates of $\mathbb{R}^{l_\alpha}$. Suppose we have
\begin{equation*}
    (f_1,f_2,f_3) \in \text{Ker} \biggl( \oplus_{\alpha} C^0(i_{A,l}, \R^k)(U_{\alpha}) \to  C^0(i_{A,l}, \R^k)(U) \biggr).
\end{equation*}
Then $f := \sum_{i=1}^3 f_i$ is the zero function over $U$. We consider the case where $\alpha = 1$. The other cases follow analogously. The fact that $f = 0$ implies that $f_1$ is independent from the variable $y_i^1$ if $y_i^1$ does not lie in the image of the following two projection maps:
\begin{align*}
    res_{U_1, U_1 \cap U_2}: i_{A,l}(U_1) &\to i_{A,l}(U_1 \cap U_2) \\
    (y_1^1, y_2^1, \cdots, y_{l_1}^1) &\mapsto (y_1, \cdots, y_{\sum_{x_i \in U_1 \cap U_{2}} l_i}), \\
    res_{U_1, U_1 \cap U_3}: i_{A,l}(U_1) &\to i_{A,l}(U_1 \cap U_3) \\
    (y_1^1, y_2^1, \cdots, y_{l_1}^1) &\mapsto (y_1, \cdots, y_{\sum_{x_i \in U_1 \cap U_3} l_i}).
\end{align*}
Therefore, $f_1$ is a continuous function depending only on variables lying in the image of $res_{U_1, U_1 \cap U_2}$ and $res_{U_1, U_1 \cap U_3}$. Analogous statements can be achieved for functions $f_2$ and $f_3$. 

Taking this new fact into consideration, we use $f = 0$ again to obtain that $f_1$ is a sum of two functions $f_{1,2} + f_{1,3}$, where $f_{1,2}$ is a function in $C^0(i_{A,l},\R^k)(U_1 \cap U_2)$, and $f_{1,3}$ is a function in $C^0(i_{A,l},\R^k)(U_1 \cap U_3)$. If not, then $f_1$ has a summand $g$ which is a non-zero function not in $C^0(i_{A,l},\R^k)(U_2)$ and $C^0(i_{A,l},\R^k)(U_3)$. Since the function $g$ cannot be canceled out with respect to taking the operation $f_1 + f_2 + f_3$, we obtain a contradiction that $f = 0$. Likewise, one can obtain that $f_2$ is a sum of two functions $f_{2,1} + f_{2,3}$, and $f_3$ is a sum of two functions $f_{3,1} + f_{3,2}$, where $f_{i,j}$'s are functions in $C^0(i_{A,l},\R^k)(U_i \cap U_j)$. Because $f = 0$, we have $f_{1,2} = -f_{2,1}$, $f_{1,3} = -f_{3,1}$, and $f_{2,3} = -f_{3,2}$. Therefore, the tuple $(f_1,f_2,f_3)$ can be rewritten as $\sum_{(i,j) \in \{(1,2),(1,3),(2,3)\}} i_{U_i \cap U_j, U_i}(f_{i,j}) - i_{U_i \cap U_j, U_j}(f_{i,j})$, proving the desired claim.
\end{remark}

\section{Deep Learning Techniques} \label{section:deep_learning}
\subsection{Functorial Interpretation}
The dual nature of skyscraper sheaf and cosheaf as both a covariant and a contravariant functor is the unique property difficult to impose on arbitrary sheaves. The dual relation of such functors can be described using the di-natural transformation between a covariant and a contravariant functor.
\begin{definition}
Let $\mathcal{C}, \mathcal{D}$ be abelian categories. Let $\mathcal{F}: \mathcal{C} \to \mathcal{D}$ be a covariant functor, and let $\mathcal{G}: \mathcal{C}^{op} \to \mathcal{D}$ be a contravariant functor. A di-natural transformation $\eta: \mathcal{F} \to \mathcal{G}$ is a family of morphisms $\{\eta_x: \mathcal{F}(x) \to \mathcal{G}(x)\}_{x \in \text{Obj}(\mathcal{C})}$ such that for every morphism $\{x \to y \} \in \text{Mor}(\mathcal{C})$, the following diagram commutes:
\begin{center}
\begin{tikzcd}
\mathcal{F}(x) \arrow[r, "\eta_x"] \arrow[d]
& \mathcal{G}(x) \\
\mathcal{F}(y) \arrow[r, "\eta_y" ]
& \mathcal{G}(y) \arrow[u]
\end{tikzcd}
\end{center}
In a similar manner, one can also define a di-natural transformation $\eta: \mathcal{G} \to \mathcal{F}$.
\end{definition}

\begin{example}
Fix a point $x \in X$. Denote by $S_x$ the skyscraper cosheaf at $x$, and denote by $S_x^{op}$ the skyscraper sheaf at $x$. The collection of identity functions $\{id_U: S_x(U) \to S_x^{op}(U)\}_{U \subset X \text{, open }}$ defines a di-natural transformation $id: S_x \to S_x^{op}$. Likewise, there exists a di-natural transformation $id: C^0(i_{A,l}^{op},\R^k) \to C^0(i_{A,l},\R^k)$.
\end{example}

Using these mathematical formulations, we define the deep learning technique as a global section of $C^{0}(i_{A,l}, \R^k)$, both as a presheaf and a copresheaf, over a locally compact topological space $X$.
\begin{definition}[Discrete Deep Learning Technique] \label{def:ml}
Let $X$ be a locally compact topological space. Fix a finite set of points $\chi_X := \{x_1,x_2,\cdots,x_N\} \subset X$. Consider a sequence of collections of finite open subsets
\begin{equation}
    \left\{ \{U_{\alpha_0}^0\}_{\alpha_0=1}^N, \{U_{\alpha_1}^1\}_{\alpha_1 \in \mathcal{A}_1}, \cdots, \{U_{\alpha_m}^m\}_{\alpha_m \in \mathcal{A}_m}, X \right\}
\end{equation}
such that for any $i$, $x_i \in U_{i}^0$ and $x_j \not\in U_j^0$ if $j \neq i$. (Note that each collection does not necessarily have to be an open cover of $X$). Let $S_{x_i}$ be a skyscraper sheaf (or cosheaf) at the point $x_i \in X$ of real vector space of dimension $k_0$.

A representation $DL_\nu^m$ obtained from a discrete deep learning algorithm with $m$ layers equipped with a fixed collection of pointwise deviations $\nu := \{\nu_i: \R^{k_0} \to \R^{k_0}\}_{i=1}^N$ at $x_i$'s is a global section of $C^{0}(i_{A,l},\R^k)$ obtained from the following procedure.
\begin{itemize}
    \item We identify $DL_\nu^m$, as an element of the copresheaf $C^{0}(i_{A,l},\R^k)$, with the image of the element $id + \nu := (id + \nu_1, \cdots, id + \nu_N)$ under the following composition of functions.
    \begin{tiny}
    \begin{align} \label{eq:ML_layers}
    \begin{split}
        \prod_{\alpha_0=1}^N C^0(i_{A,l},\R^{k_0})(U_{\alpha_0}^0) \to \prod_{\alpha_1 \in \mathcal{A}_1} C^0(i_{A,l},\R^{k_1})(U_{\alpha_1}^1) \to &\cdots \to \prod_{\alpha_m \in \mathcal{A}_m} C^0(i_{A,l},\R^{k_m})(U_{\alpha_m}^m) \to C^0(i_{A,l},\R^{k})(X) \\
        id + \nu := (id + \nu_1, \cdots, id + \nu_N) \mapsto (f_{1,1}, f_{1,2}, \cdots, f_{1,|\mathcal{A}_1|}) \mapsto &\cdots \mapsto (f_{m,1}, f_{m,2}, \cdots, f_{m,|\mathcal{A}_m|}) \mapsto DL_\nu^m
    \end{split}
    \end{align}
    \end{tiny}
    
    \item Each function, possibly non-linear,
    \begin{equation}
        \psi_{i+1}: \prod_{\alpha_i \in \mathcal{A}_i} C^0(i_{A,l},\R^{k_i})(U_{\alpha_i}^i) \to \prod_{\alpha_{i+1} \in \mathcal{A}_{i+1}} C^0(i_{A,l},\R^{k_{i+1}})(U_{\alpha_{i+1}}^{i+1})
    \end{equation}
    corresponds to the $i+1$-th layer of the discrete deep learning technique. We denote by $(f_{i,1}, f_{i,2}, \cdots, f_{i,|\mathcal{A}_i|})$ the image obtained after applying $i$ layers of the discrete deep learning technique, i.e.
    \begin{equation}
        (f_{i,1}, f_{i,2}, \cdots, f_{i,|\mathcal{A}_i|}) = (\psi_i \circ \psi_{i-1} \circ \cdots \circ \psi_0) (id + \nu_1, \cdots, id + \nu_N)
    \end{equation}
    
    \item We identify $DL_\nu^m$ as a global section of the presheaf $C^0(i_{A,l}^{op}, \R^k)$ under the composition of the di-natural transformation.
    \begin{align}
    \begin{split}
        C^0(i_{A,l}^{op},\R^k)(X) \to  &C^0(i_{A,l},\R^k)(X) \\
        DL_\nu^m \mapsto &id(DL_\nu^m)
    \end{split}
    \end{align}
\end{itemize}
A discrete deep learning algorithm $DL^m$ composed of $m$ layers is a function
\begin{equation}
    \prod_{\alpha_0 = 1}^N C^0(i_{A,l}, \R^{k_0})(U_{\alpha_0}^0) \to C^0(i_{A,l}, \R^k)(X)
\end{equation}
obtained from the aforementioned procedure.
\end{definition}

\begin{definition} [Neighborhood aggregating axioms] \label{def:neighborhood_aggregate}
A neighborhood aggregating layer $\psi_n$ of a discrete deep learning technique $DL_\nu^m$ is a function 
\begin{equation*}
    \psi_n: \prod_{\alpha_{n-1} \in \mathcal{A}_{n-1}} C^0(i_{A,l}^{op}, \mathbb{R}^{k_{n-1}})(U_{\alpha_{n-1}}^{n-1}) \to \prod_{\alpha_{n} \in \mathcal{A}_{n}} C^0(i_{A,l}^{op}, \mathbb{R}^{k_{n}})(U_{\alpha_{n}}^n)
\end{equation*}
whose associated collections of finitely many open subsets
\begin{equation*}
    \{U_{\alpha_{n-1}}^{n-1}\}_{\alpha_{n-1} \in \mathcal{A}_{n-1}}, \{U_{\alpha_{n}}^{n}\}_{\alpha_{n} \in \mathcal{A}_{n}}
\end{equation*}
satisfy the following four neighborhood aggregating axioms.
\begin{enumerate} \label{equation:deep_learning_axioms}
    \item \textbf{Locality}: There exists a point $x_j \in \chi_X$ such that $x_j \not\in U_{\alpha_n}^n$ for every $\alpha_n \in \mathcal{A}_n$.
    \item \textbf{Strictness}: $\# \mathcal{A}_{n-1} > \# \mathcal{A}_{n}$.
    \item \textbf{Non-triviality}: For each $\alpha_n \in \mathcal{A}_n$, there exists a proper subset $\widetilde{\mathcal{A}}_{\alpha_n, n-1} \subsetneq \mathcal{A}_n$ such that
    \begin{equation}
    U_{\alpha_n}^n = \bigcup_{\alpha_{n-1} \in \widetilde{\mathcal{A}}_{\alpha_n, n-1}} U_{\alpha_{n-1}}^{n-1}
    \end{equation}
    \item \textbf{Distinctness}: For any $\alpha_{i}, \alpha_j \in \mathcal{A}_n$, 
    \begin{equation}
        U_{\alpha_i}^n \neq U_{\alpha_j}^n
    \end{equation}
\end{enumerate}
\end{definition}

\begin{remark}
We note that axioms (2) and (3) imply that there exists an $\alpha_n \in \mathcal{A}_n$ whose any proper subset $\widetilde{A}_{\alpha_n,n-1}$ satisfying axiom (3) is of size at least $2$.
\end{remark}

\begin{remark}
The set of pointwise deviations $\{\nu_i: \R^{l_i} \to \R^{l_i}\}_{i=1}^N$ can be considered as either sensor noises in detecting signals at each point $x_i \in X$, or pointwise differences between two input data defined over $X$. One can hence regard the collection of real representations obtained from applying a discrete deep learning algorithm $DL^m$ to a given dataset as the image
\begin{equation}
    DL^m(\{id + \nu\}_{\nu \in \mathcal{V}})
\end{equation}
for some collections of deviations $\mathcal{V}$ such that every element $\nu \in \mathcal{V}$ is comprised of constant functions.
\end{remark}

\begin{remark}
Using the restriction morphism for the presheaf $C^0(i_{A,l}^{op}, \R^k)$, one may define a representation $DL_{\nu,U}^m$ of an open subset $U \subset X$ induced from the representation $DL_\nu^m$ of $X$ defined as
\begin{equation}
    DL_{\nu,U}^m := res_{X,U} \circ DL_\nu^m
\end{equation}
We note that the representations $DL_{\nu,U}^m$ are not necessarily identical to those resulting from imposing neural layers of $DL_\nu^m$ (\ref{eq:ML_layers}).
\end{remark}

\begin{figure}
            \centering
            \includegraphics[width=120mm]{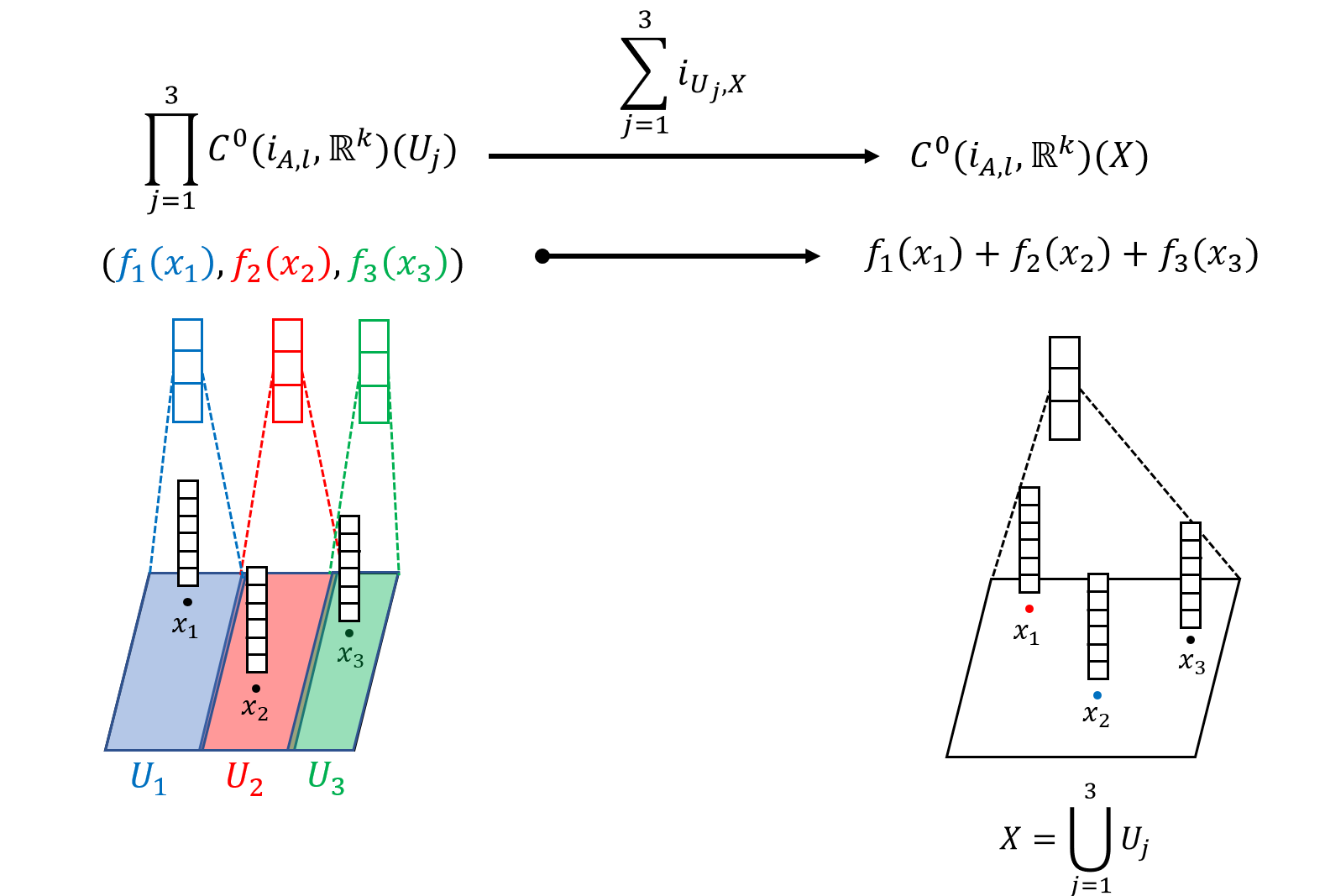}
            \caption{Summation by inclusion maps, which defines the surjectivity condition of the cosheaf axiom}
\end{figure}
\begin{figure}
            \centering
            \includegraphics[width=120mm]{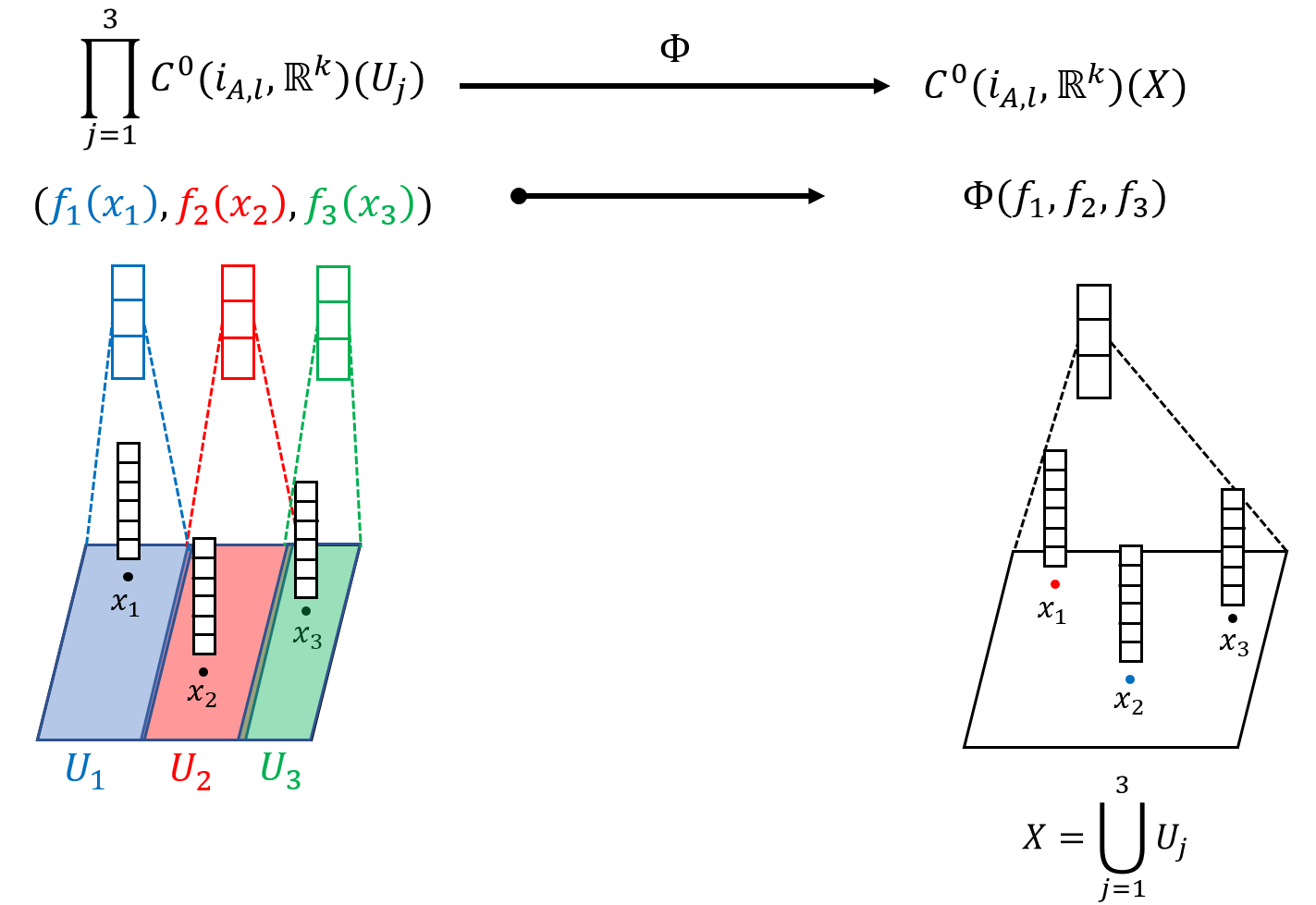}
            \caption{An example of a layer of a convolutional (or message passing) neural network}
\end{figure}

\begin{definition}[Factorization through inclusion] \label{definition:fact_inc}
Let $DL^m$ be a discrete deep learning algorithm over $X$ composed of $m$ layers. We say that the $i$-th layer satisfying axiom (3) of the neighborhood aggregating axiom 
\begin{equation}
    \psi_i: \prod_{\alpha_i \in \mathcal{A}_i} C^0(i_{A,l},\R^{k_i})(U_{\alpha_i}^i) \to \prod_{\alpha_{i+1} \in \mathcal{A}_{i+1}} C^0(i_{A,l},\R^{k_{i+1}})(U_{\alpha_{i+1}}^{i+1})
\end{equation}
factors through the inclusion map if there exists a function for each $\alpha_i \in \mathcal{A}_i$
\begin{equation}
    \varphi_{\alpha_i}: C^0(i_{A,l},\R^{k_i})(U_{\alpha_i}^i) \to C^0(i_{A,l},\R^{k_{i+1}})(U_{\alpha_{i}}^{i}),
\end{equation}
and a collection of continuous activation functions
\begin{equation}
    F_{\alpha_{i+1}}: C^0(i_{A,l},\R^{k_{i+1}})(U_{\alpha_{i+1}}^{i+1}) \to C^0(i_{A,l},\R^{k_{i+1}})(U_{\alpha_{i+1}}^{i+1})
\end{equation}
such that 
\begin{equation}
    \psi_i = \left( F_{\alpha_{i+1}} \circ \sum_{\alpha_j \in \mathcal{A}_i} i_{U_{\alpha_i}, U_{\alpha_{i+1}}} \circ \prod_{\alpha_i \in \mathcal{A}_i}\varphi_{\alpha_i} \right)_{\alpha_{i+1} \in \mathcal{A}_{i+1}}
\end{equation}
\end{definition}

\begin{remark}
The definition provided above generalizes the sum-decomposition of continuous functions $f: \mathbb{R}^n \to \mathbb{R}$ proposed in \cite{ZK17} and \cite{WFE19}. A function $f: \mathbb{R}^n \to \mathbb{R}$ is said to be sum-decomposable if there exist functions $\phi: \mathbb{R} \to Z$ and $\rho: Z \to \mathbb{R}$ for some topological space $Z$ such that
\begin{equation}
    f(x_1,\cdots,x_n) = \rho(\sum_{i=1}^n \phi(x_i))
\end{equation}
\end{remark}

\begin{example}
Any fully connected layer $L_i$ with non-linear activation functions factors through the inclusion map.
\end{example}

\begin{remark} \label{remark:comm_diag}
Using the di-natural transformation $id: C^0(i_{A,l}, \R^k) \to C^0(i_{A,l}^{op}, \R^k)$, it holds that the following diagram commutes for any neighborhood aggregating $i$-th layer $\psi_i$ that factors through inclusion.
\begin{tiny}
\begin{center}
\begin{tikzcd}
& \prod_{\alpha_i \in \mathcal{A}_i} C^0(i_{A,l},\R^{k_{i+1}})(U_{\alpha_i}^i) \arrow[r] \arrow[dd, anchor=center, xshift=-20pt, "\left( F_{\alpha_{i+1}} \circ \sum_{\alpha_i} i_{U_{\alpha_i}, U_{\alpha_{i+1}}} \right)_{\alpha_{i+1}}"] & \prod_{\alpha_i \in \mathcal{A}_i} C^0(i_{A,l}^{op},\R^{k_{i+1}})(U_{\alpha_i}^i) \\
\prod_{\alpha_i \in \mathcal{A}_i} C^0(i_{A,l},\R^{k_i})(U_{\alpha_i}^i) \arrow[ur, dotted, "\prod_{\alpha_i \in \mathcal{A}_i} \varphi_{\alpha_i}"] \arrow[dr, "\psi_i"]& & \\
& \prod_{\alpha_{i+1} \in \mathcal{A}_{i+1}} C^0(i_{A,l},\R^{k_{i+1}})(U_{\alpha_{i+1}}^{i+1}) \arrow[r, "id" ] & \prod_{\alpha_{i+1} \in \mathcal{A}_{i+1}} C^0(i_{A,l}^{op},\R^{k_{i+1}})(U_{\alpha_{i+1}}^{i+1}) \arrow[uu, anchor=center, xshift=40pt, "\left( \sum_{\alpha_{i+1}} \text{res}_{U_{\alpha_{i+1}}, U_{\alpha_{i}}} \right)_{\alpha_i}"]
\end{tikzcd}
\end{center}
\end{tiny}

\end{remark}

\subsection{Architectural Limitations}
We have all the key ingredients to prove a number of empirically verified architectural limitations of discrete deep learning techniques obtained from analyzing a collection of data sets, as aforementioned in the introduction. The following two theorems state that discrete deep learning techniques with non-linear layers may misidentify the characteristics of a collection of data. Interestingly, the factorability of the neighborhood aggregating $i$-th layer of the discrete deep learning algorithm affects whether the architecture is subject to non-unique gluing of local data or adversarial attacks.

\begin{theorem} [Non-unique local explainability] \label{theorem:first_limitation}
Let $DL^m$ be a discrete deep learning algorithm over $X$ composed of $m$ layers such that at least one of the layers is non-linear or does not factor through the inclusion map. Let $\{U_{\alpha}\}_\alpha$ be any open cover of $X$ such that there exists a point $x_\alpha \in \{x_i\}_{i=1}^N$ such that $x_\alpha \not\in U_\alpha$. Then there exist deviations $\mu, \nu$ such that $DL_\mu^m \neq DL_\nu^m$ whereas for every $\alpha$, $DL_{\mu,U_\alpha}^m = DL_{\nu,U_\alpha}^m$ for every $\alpha$.
\end{theorem}
\begin{proof}
The condition that there exists a layer $\psi_i$ which is non-linear or does not factor through the inclusion map implies that there exists a deviation $\nu$ such that \begin{equation}
    DL_\nu^m \in C^0(i_{A,l}^{op},\R^k)(X) \setminus \text{Hom}(i_{A,l}^{op}, \R^k)(X).
\end{equation}
Here we use the fact that elements in $\text{Hom}(i_{A,l}^{op}, \R^k)(X)$ are linear maps between $\R$-vector spaces. Choose an open cover $\{U_{\alpha_i}\}_{\alpha_i \in \mathcal{A}_i}$ used from the $i$-th layer of the discrete deep learning algorithm. Proposition \ref{prop:sheaf_fail} demonstrates that the morphism
\begin{equation}
    C^0(i_{A,l}^{op}, \R^k)(X) \to C^0(i_{A,l}, \R^k)(X) \to \prod_\alpha C^0(i_{A,l}, \R^k)(U_\alpha)
\end{equation}
is not injective, whereas the presheaf $\text{Hom}(i_{A,l}^{op}, \R^k)(X)$ satisfies the sheaf axioms.
\end{proof}
\begin{remark}
As shown in Remark \ref{remark:sheaf_fail}, the discrete deep learning algorithms from Theorem \ref{theorem:first_limitation} fails the locality condition (\ref{eqn:fail_condition}) but satisfies the gluing condition (\ref{eqn:succeed_condition}) given a finite collection of open cover $\{U_\alpha\}_\alpha$ of $X$. Figures \ref{fig:first_limitation_fail} and \ref{fig:first_limitation_succeed} illustrates how such discrete deep learning algorithms fail and satisfy the sheaf axioms given a finite collection of open sets. 
\end{remark}

\begin{figure}
                \centering
                \includegraphics[width=90mm]{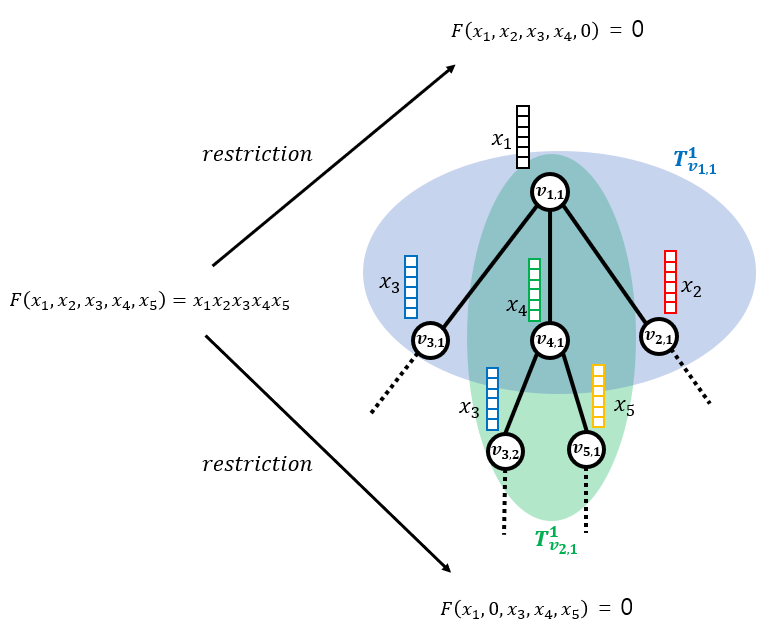}
                \caption{Presheaf: Does not satisfy the locality condition}
                \label{fig:first_limitation_fail}
\end{figure}
\begin{figure}
                \centering
                \includegraphics[width=135mm]{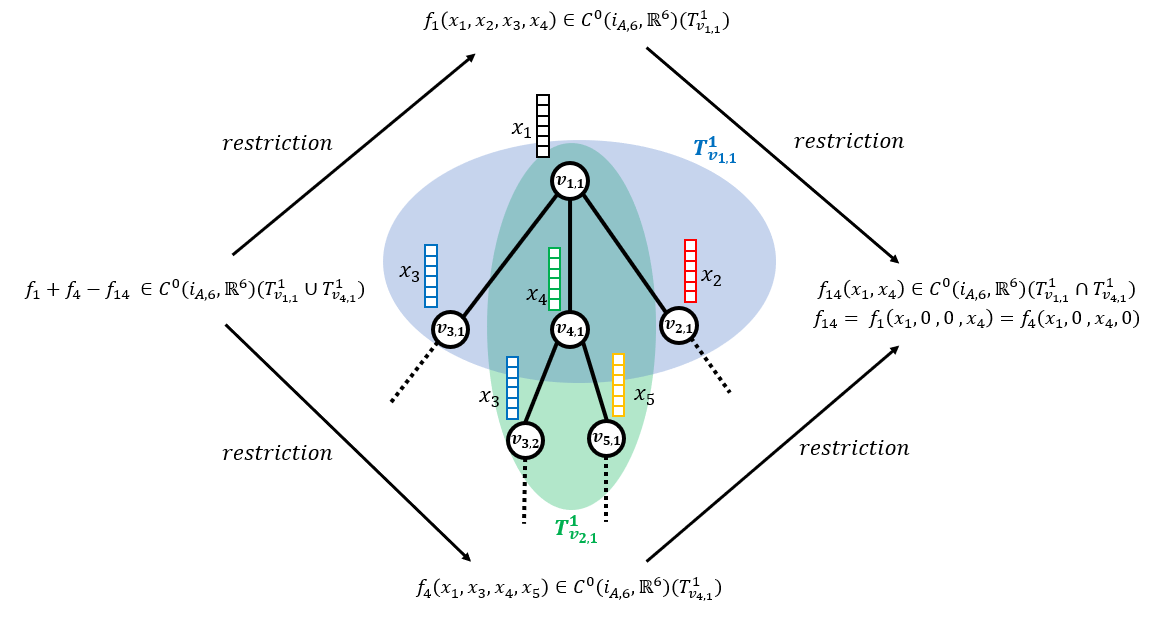}
                \caption{Presheaf: Satisfies the gluing condition}
                \label{fig:first_limitation_succeed}
\end{figure}

\begin{theorem} [Adversarial Attack]\label{theorem:second_limitation}
    Let $DL^m$ be a discrete deep learning algorithm over $X$ composed of $m$ layers. Suppose there exists $1 \leq j \leq m$ such that the $j$-th layer $\psi_j$ satisfies axioms (2) and (3) of the neighborhood aggregating axiom and factors through the inclusion map. Then for any $\delta > 0$, there exist infinitely many products of local sections
    \begin{equation}
        g := (g_1, g_2, \cdots, g_{|\mathcal{A}_j|}) \in \prod_{\alpha_j \in \mathcal{A}_j} C^0(i_{A,l}^{op}, \R^{k_{j+1}})(U_{\alpha_j}^j)
    \end{equation}
    such that for any $f \in \prod_{\alpha_0 \in \mathcal{A}_0} C^0(i_{A,l}^{op}, \mathbb{R}^{k_0})(U_{\alpha_0}^0)$,
    \begin{equation}
        (\psi_{j+1} \circ \cdots \circ \psi_1) \circ (f) =\left( F_{\alpha_{i+1}} \circ \sum_{\substack{\alpha_j \in \mathcal{A}_j }} i_{U_{\alpha_j}, U_{\alpha_{j+1}}} \right)_{\alpha_{j+1} \in \mathcal{A}_{j+1}} \circ (g)
    \end{equation}
    whereas for any vector $v \in \R^{k_0 \times N}$ and $p > 0$,
    \begin{equation}
        \left| \left( \left(\prod_{\alpha_i \in \mathcal{A}_i} \varphi_{\alpha_i} \circ \psi_{j} \circ \cdots \circ \psi_1 \right) \circ (f) \right)(v) - g(v) \right|_p > \delta
    \end{equation}
\end{theorem}
\begin{proof}
We observe that the morphism
\begin{equation}
     \left( \sum_{\alpha_j \in \mathcal{A}_j} i_{U_{\alpha_j}, U_{\alpha_{j+1}}} \right)_{\alpha_{j+1} \in \mathcal{A}_{j+1}}: \prod_{\alpha_j \in \mathcal{A}_j} C^0(i_{A,l}^{op}, \R^{k_{j+1}})(U_{\alpha_j}) \to \prod_{\alpha_{j+1} \in \mathcal{A}_{j+1}}C^0(i_{A,l}^{op}, \R^{k_{j+1}})(U_{\alpha_{j+1}})
\end{equation}
from (\ref{equation:presheaf_map}) is not injective. For each open neighborhood $U_{\alpha_j}$, let $f_{\alpha_j} \in C^0(i_{A,l}^{op}, \R^{k_{j}})(U_{\alpha_j})$ be a local section over $U_{\alpha_j}$. Given a vector $v_{\alpha_j} \in \R^{k_{j+1}}$, we denote by $\varphi_{\alpha_j} \circ f_{\alpha_j} + v_{\alpha_j}$ the function
\begin{equation}
    (\varphi_{\alpha_j} \circ f_{\alpha_j} + v_{\alpha_j}) (y_1, y_2, \cdots, y_{k_{j+1}}) = \varphi_{\alpha_j} (f_{\alpha_j}(y_1, y_2, \cdots, y_{k_{j+1}})) + v_{\alpha_j},
\end{equation}
where $\varphi_{\alpha_j}$ is the inclusion map for the index $\alpha_j \in \mathcal{A}_j$ from Definition \ref{definition:fact_inc}.
For each $\alpha_{j+1}$, there exists any set of vectors $\{m_\alpha\}_\alpha \subset \R^{k_{j+1}}$ such that 
\begin{equation} \label{equation:adv_attack_constraint}
    \sum_{\substack{\alpha_j \in \mathcal{A}_j \\ U_{\alpha_j} \subset U_{\alpha_{j+1}}}} m_{\alpha_j} = 0.
\end{equation}
Hence the following equation holds:
\begin{equation}
    \sum_{\alpha_j \in \mathcal{A}_j} i_{U_{\alpha_j}, U_{\alpha_{j+1}}} \circ (\varphi_{\alpha_j} \circ f_{\alpha_j} + m_{\alpha_j}) = \sum_{\alpha_j \in \mathcal{A}_j} i_{U_{\alpha_j}, U_{\alpha_{j+1}}} \circ \varphi_{\alpha_j} \circ f_{\alpha_j}.
\end{equation}
We hence obtain that
\begin{align}
\begin{split}
    \psi_{j+1} \circ \left( \prod_{\alpha_j \in \mathcal{A}_j} f_{\alpha_j} \right) &= \left( F_{\alpha_{j+1}} \circ \sum_{\alpha_j \in \mathcal{A}_j} i_{U_{\alpha_j}, U_{\alpha_{j+1}}} \circ \prod_{\alpha_j \in \mathcal{A}_j} \varphi_{\alpha_j} \right)_{\alpha_{j+1} \in \mathcal{A}_{j+1}} \circ \left( \prod_{\alpha_j \in \mathcal{A}_j} f_{\alpha_j} \right) \\ 
    &= \left( F_{\alpha_{j+1}} \circ \sum_{\alpha_j \in \mathcal{A}_j} i_{U_{\alpha_j}, U_{\alpha_{j+1}}} \right)_{\alpha_{j+1} \in \mathcal{A}_{j+1}} \circ \left( \prod_{\alpha_j \in \mathcal{A}_j} \left( \varphi_{\alpha_j} \circ f_{\alpha_j} + m_{\alpha_j} \right) \right)
\end{split}
\end{align}
Because $\# \mathcal{A}_j > \# \mathcal{A}_{j+1}$, the subspace of vectors $\{m_{\alpha_j}\}_{\alpha_j \in \mathcal{A}_j}$ that satisfies (\ref{equation:adv_attack_constraint}) for every $\alpha_{j+1} \in \mathcal{A}_{j+1}$ has positive dimension. For such vectors $\{m_{\alpha_j}\}_{\alpha_j \in \mathcal{A}_j}$, the following equation holds for any vector $v \in \mathbb{R}^{k_0 \times N}$ and $p > 0$.
\begin{equation}
    \left| \prod_{\alpha_j \in \mathcal{A}_j} (\varphi_{\alpha_j} \circ f_{\alpha_j})(v) - \prod_{\alpha_j \in \mathcal{A}_j}  \left( \varphi_{\alpha_j} \circ f_{\alpha_j} + m_{\alpha_j} \right) (v) \right|_p = \left( \sum_{\alpha_j \in \mathcal{A}_j} |m_{\alpha_j}|^p \right)^{\frac{1}{p}}.
\end{equation}
The statement of the theorem follows by inductively defining $\prod_{\alpha_j \in \mathcal{A}_j}f_{\alpha_j}$ with respect to $j$, setting $g = \prod_{\alpha_j \in \mathcal{A}_j}(\varphi_{\alpha_j} \circ f_{\alpha_j} + m_{\alpha_j})$, and choosing vectors $\{m_{\alpha_j}\}_{\alpha_j \in \mathcal{A}_j}$ whose $\ell_p$ norm is greater than $\delta$.
\end{proof}

We now prove that under certain constraints on the final layer of the discrete deep learning algorithm, the state-of-the-art performances in analyzing all collections of data may not be achievable from a predetermined discrete deep learning architecture.
\begin{theorem}[Dataset Dependency]\label{theorem:third_limitation}
Let $DL^m$ be a discrete deep learning algorithm over $X$ composed of $m$ layers such that the last layer $\psi_m$ is a neighborhood aggregating layer that factors through inclusion, i.e.
\begin{equation}
    \psi_m = F \circ \left( \sum_{\alpha_{m} \in \mathcal{A}_m} i_{U_{\alpha_m},X} \circ \prod_{\alpha_m \in \mathcal{A}_m} \varphi_{\alpha_m} \right)
\end{equation}
Suppose further that the predetermined continuous function $F: \mathbb{R}^k \to \mathbb{R}^k$ satisfies either one of the following conditions:
\begin{enumerate}
    \item $F$ is not surjective.
    \item $F$ is not open.
    \item $F$ is open and bijective.
\end{enumerate}
Then there exists a function $f \in C^0(i_{A,l}, \R^k)(X)$ such that for every deviations $\nu$,
\begin{equation}
    f \neq DL_\nu^m.
\end{equation}
\end{theorem}

\begin{proof}
Suppose the last layer $\psi_m$ of the discrete deep learning algorithm factors through inclusion, as shown in Definition \ref{definition:fact_inc}. Observe that any function
\begin{equation}
    F: \mathbb{R}^k \to \mathbb{R}^k
\end{equation}
induces a function of sets
\begin{align} \label{equation:ind_func}
\begin{split}
    F:  C^0(i_{A,l}, \mathbb{R}^k)(X) &\to C^0(i_{A,l}, \mathbb{R}^k)(X) \\
    g &\mapsto F \circ g
\end{split}
\end{align}
If $F$ is not surjective or not open, then the induced function from (\ref{equation:ind_func}) is not surjective because the set-theoretical right inverse of $F$ either does not exist or is not continuous. If $F$ is open and bijective, then $F$ is a homeomorphism. Observe that (\ref{equation:ind_func}) is a bijective function. But because $\psi_m$ factors through inclusion, $\psi_m$ is not surjective. Indeed, from the proof of Proposition \ref{prop:sheaf_fail}, the morphism
\begin{equation}
    \left( \sum_{\alpha_m \in \mathcal{A}_m} i_{U_{\alpha_m}, X} \right): \oplus_{\alpha} C^0(i_{A,l}^{op}, \R^k)(U_\alpha) \to C^0(i_{A,l}^{op}, \R^k)(X)
\end{equation}
from (\ref{equation:presheaf_map}) is not surjective.
\end{proof}

\begin{remark}
We note that any discrete deep learning algorithm whose last layer is a fully connected layer with non-linear activation functions satisfy either one of the conditions provided above. Indeed, any function $f:\mathbb{R} \to \mathbb{R}$ which is open and surjective is in fact a homeomorphism. Such result, however, is not necessarily the case for any continuous functions $f: \mathbb{R}^k \to \mathbb{R}^k$ for any $k \geq 3$. In fact, there always exists a surjective continuous open function $F: \mathbb{R}^m \to \mathbb{R}^n$ for any $n \geq m \geq 3$ which is not a homeomorphism if $n > m$. This seemingly surprising fact is a result of Whitehead's theorem \cite{Ha02} and John Walsh's results on the existence of surjective open continuous maps between manifolds whose induced morphism on their fundamental groups are surjective \cite{Wa75}, see for instance the proof provided by Moishe Kohan on \cite{ko19}.

The technical condition shows that as long as $F: \mathbb{R}^k \to \mathbb{R}^k$ is not a continuous open surjection, then the discrete deep learning algorithm whose last layer factors through inclusion cannot achieve state-of-the-art performance in analyzing all collections of data. If, however, $F: \mathbb{R}^k \to \mathbb{R}^k$ is an open surjection which is not injective, then the function of sets $F: C^0(i_{A,l}^{op}, \mathbb{R}^k)(X) \to C^0(i_{A,l}^{op}, \mathbb{R}^k)(X)$ is surjective but not injective, because the right inverse of $F$ is continuous. If it is the case that
\begin{equation}
    F^{-1} \left( C^0(i_{A,l}^{op}, \mathbb{R}^k)(X) \right) = \sum_{\alpha_m \in \mathcal{A}_m} i_{U_{\alpha_m}, X} \left( C^0( i_{A,l}^{op}, \mathbb{R}^k)(U_{\alpha_m}) \right)
\end{equation}
then the function of sets $F: C^0(i_{A,l}^{op}, \mathbb{R}^k)(X) \to C^0(i_{A,l}^{op}, \mathbb{R}^k)(X)$ is surjective, providing a counterexample to Theorem \ref{theorem:third_limitation}. 
\end{remark}
\begin{remark}
Similar to Remark \ref{remark:sheaf_fail}, the neighborhood aggregating discrete deep learning algorithms from Theorem \ref{theorem:second_limitation} fails the surjectivity condition (\ref{eqn:fail_condition}) but satisfies the gluing condition (\ref{eqn:succeed_condition}) given a finite collection of open cover $\{U_\alpha\}_\alpha$ of $X$. Figures \ref{fig:second_limitation_fail} and \ref{fig:second_limitation_succeed} illustrates how such neighborhood aggregating algorithms fail and satisfy the sheaf axioms given a finite collection of open sets. 
\end{remark}
\begin{figure}
    \centering
    \includegraphics[width=90mm]{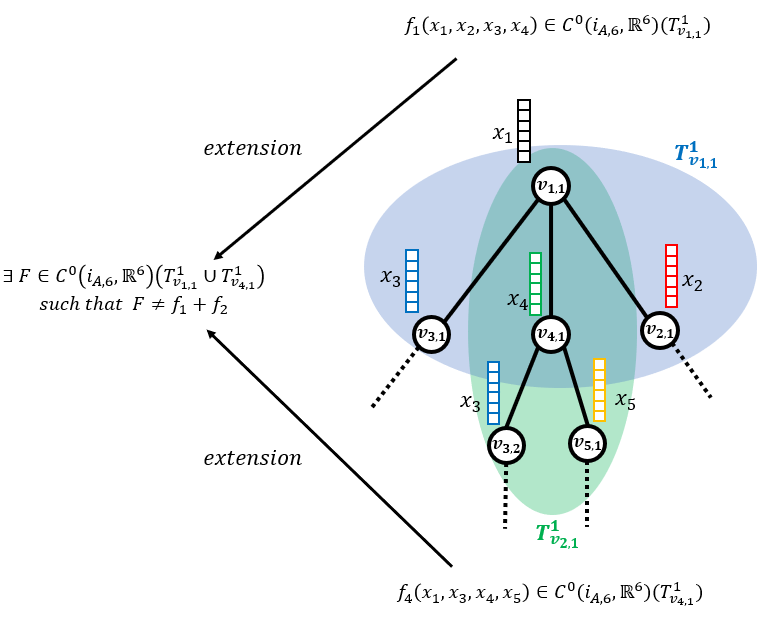}
    \caption{Copresheaf: Does not satisfy the surjectivity condition}
    \label{fig:second_limitation_fail}
\end{figure}
\begin{figure}
    \centering
    \includegraphics[width=135mm]{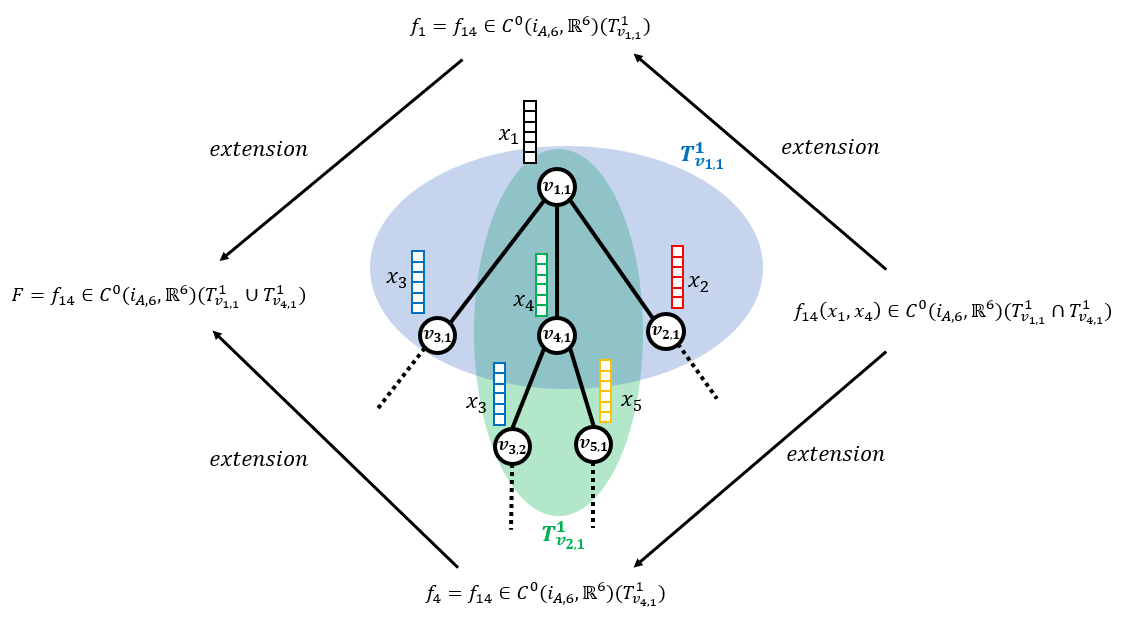}
    \caption{Copresheaf: Satisfies the gluing condition}
    \label{fig:second_limitation_succeed}
\end{figure}

\subsection{Cohomological Properties}
The aforementioned sections demonstrate that the architectural limitations suggested in Theorems \ref{theorem:first_limitation}, \ref{theorem:second_limitation}, and \ref{theorem:third_limitation} originate from the property that the presheaf $C^0(i_{A,l},\R^k)$ (and the cosheaf) fails to be a sheaf (and the cosheaf). Nevertheless, there is a canonical method - the sheafification of a presheaf - which constructs a sheaf associated to a presheaf. This construction would then allow neighborhood aggregating deep learning algorithms to avoid Theorem \ref{theorem:first_limitation}, whose proof relied on the fact that the presheaf of our interest is not a sheaf. 
\begin{definition}
Let $\mathcal{F}$ be a presheaf of abelian groups over $X$. The stalk of the sheaf at a point $x \in X$ is the direct limit of the groups $\mathcal{F}(U)$ for all open sets $U$ containing $x$.
\begin{equation}
    \mathcal{F}_x := \varinjlim_{x \in U} \mathcal{F}(U)
\end{equation}
\end{definition}

\begin{definition}[\cite{Ha77}, Proposition-Definition 2.1.2]
Let $\mathcal{F}$ be a presheaf over $X$. Define a contravariant functor $\mathcal{F}^+$ from the category of open subsets of $X$ to the category of abelian groups as follows.
\begin{equation}
    \mathcal{F}^+(U) := \{s: U \to \bigsqcup_{x \in U} \mathcal{F}_x \text{ such that } (*)\}
\end{equation}
where the conditions $(*)$ are:
\begin{enumerate}
    \item For each $x \in U$, $s(x) \in \mathcal{F}_x$ 
    \item For each $x \in U$, there exists an open neighborhood $V$ containing $x$ and a section $t \in \mathcal{F}(V)$ such that for all $y \in V$, $t_y = s_y$.
\end{enumerate}
Note that there exists a natural inclusion map $\theta: \mathcal{F} \to \mathcal{F}^+$ for any presheaf $\mathcal{F}$.
\end{definition}

\begin{remark}
By sheafifying the presheaf $C^0(i_{A,l}, \R^k)$, the collection of discrete deep learning techniques over a given locally compact Hausdorff topological space $X$ is a subset of the global section of the sheaf $C^{0,+}(i_{A,l}(X),\R^k)$.
\end{remark}

Now that one has constructed a sheaf $C^{0,+}(i_{A,l},\R^k)$ associated to the presheaf $C^0(i_{A,l},\R^k)$ it is a natural question to ask what cohomological data the constructed sheaf harbors. Cohomology groups of sheaves are of great importance in excavating underlying topological properties of a given space $X$. For example, the usual simplicial homology / cohomology groups of a topological space $X$ provides an easily computable mathematical machinery to encapsulate certain geometric invariants of $X$, such as the first cohomology group of a graph $G$, which contains all possible cyclic subgraphs of $G$. Unfortunately, the sheaf $C^{0,+}(i_{A,l},\R^k)$ does not harbor non-trivial higher cohomological data, as what one would expect to occur for simplicial cohomology groups. 

\begin{definition}[Flasque Sheaves and Cosheaves]
Let $\mathcal{F}$ be a sheaf of abelian groups over $X$. We say that $\mathcal{F}$ is flasque if for a pair of open subsets
\begin{equation}
    U \subset V \subset X
\end{equation}
the restriction map
\begin{equation}
    r: \mathcal{F}(V) \to \mathcal{F}(U)
\end{equation}
is surjective as group homomorphisms. Likewise, we say that a cosheaf $\widetilde{\mathcal{F}}$ is flasque if for a pair of open subsets
\begin{equation}
    U \subset V \subset X
\end{equation}
the inclusion map
\begin{equation}
    i: \widetilde{\mathcal{F}}(U) \to \widetilde{\mathcal{F}}(V)
\end{equation}
is a monomorphism.
\end{definition}

\begin{proposition}
If $\mathcal{F}$ is a flasque sheaf over $X$, then $\mathcal{F}$ is acyclic, i.e. $H^i(X,\mathcal{F}) = 0$ for all $i > 0$.
\end{proposition}
\begin{proof}
See Proposition 3.2.5 of \cite{Ha77}.
\end{proof}

\begin{theorem} \label{theorem:fourth_limitation}
Given any topological space $X$, the sheaves $C^{0,+}(i_{A,l},\R^k)$ and $\text{Hom}(i_{A,l},\R^k)$ over $X$ are flasque.
\end{theorem}
\begin{proof}
This follows from the fact that the skyscraper cosheaf $S_{x}$ over any topological space $X$ and a point $x \in X$ is flasque.
\end{proof}

In other words, there are no obstructions in gluing local sections to construct global sections over a topological space $X$, regardless of the underlying geometric or topological properties of $X$. Thus, we obtain that discrete deep learning techniques with neighborhood aggregating layers cannot effectively distinguish geometric or topological invariants of the underlying spaces $X$'s. We have therefore verified the last empirically verified property of discrete deep learning techniques: Topological data analysis or persistent homological techniques often enhances performances of such algorithms in classifying or analyzing geometric properties of image or graph data sets, because they incorporate cohomological data which cannot be derived from higher cohomological datas of the sheaves $C^{0,+}(i_{A,l},\R^k)$ and $\text{Hom}(i_{A,l},\R^k)$.

\section{Examples of Neighborhood Aggregating Discrete Deep Learning Techniques} \label{section:deep_learning_examples}

In this section, we demonstrate that a wide class of convolutional neural networks, message passing neural networks, and recurrent neural networks are discrete deep learning techniques comprised of neighborhood aggregating layers. We thereby provide a theoretical explanation on empirical results on the constructions and limitations of these discrete deep learning algorithms by using Theorems \ref{theorem:first_limitation}, \ref{theorem:second_limitation}, and \ref{theorem:third_limitation}.

\subsection{Convolutional Neural Networks}

We use the identification of discrete deep learning techniques as elements of global sections of $C^{0}(i_{A,l},\R^k)$ to redefine convolutional neural networks.

\begin{example}[Convolutional Neural Networks] \label{example:cnn}
Let $I^2 := [0,1] \times [0,1]$ be a grid endowed with the usual Euclidean metric. Fix a positive integer $N > 0$. One can define a collection of compact subsets 
\begin{equation}
    \left\{K_{m,n} := \left[ \frac{m}{N}, \frac{m+1}{N} \right] \times \left[ \frac{n}{N}, \frac{n+1}{N} \right] \right\}_{m,n=0}^{N-1}
\end{equation}
For each compact subset $K_{m,n}$, pick a point $x_{m,n}$ inside its interior. Let $A$ be the discrete set comprised of such points, endowed with the natural inclusion map $i:A \to I^2$.
\begin{equation}
    A := \{x_{m,n}\}_{m,n=0}^{N-1}.
\end{equation}
The vanilla convolutional neural network $CNN$ without sensor noises can be identified as the image of the products of the identity functions.
\begin{equation}
    CNN \in C^{0}(i_{A,l}, \R^k)(I^2)
\end{equation}
The associated sequence of finite open covers of $I^2$ is given by
\begin{equation}
    \left\{ \{U_{\alpha_0}^0\}_{\alpha_0=1}^N, \{U_{\alpha_1}^1\}_{\alpha_1 \in \mathcal{A}_1}, \cdots, \{U_{\alpha_m}^m\}_{\alpha_m \in \mathcal{A}_m}, I^2 \right\}
\end{equation}
whose elements satisfy the neighborhood aggregating axioms from Definition \ref{def:neighborhood_aggregate}. For instance, one may take $U_{\alpha_i}^i$ to be an open subset of $I^2$ which contains a collection of compact subsets
\begin{equation}
    U_{\alpha_i}^i \supset \cup_{m=i_1}^{i_2} \cup_{n=i_3}^{i_4} K_{m,n}
\end{equation}
for some integers $0 \leq i_1 < i_2 \leq N-1$ and $0 \leq i_3 < i_4 \leq N-1$.

The convolutional layer is constructed as a matrix multiplication
\begin{align}
    \begin{split}
        \prod_{\alpha_i \in \mathcal{A}_i} C^0(i_{A,l}^{op},\R^3)(U_{\alpha_i}^i) &\to \prod_{\alpha_{i+1} \in \mathcal{A}_{i+1}} C^0(i_{A,l}^{op},\R^{3})(U_{\alpha_{i+1}}^{i+1}) \\
        (f)_{\alpha_{i}} &\mapsto W \times (f)_{\alpha_i}
    \end{split}
\end{align}
where $W$ is a trainable $\R^{|\mathcal{A}_i| \times |\mathcal{A}_{i+1}|}$ matrix whose rows are all equal to a predetermined convolutional filter up to coordinate-wise permutations.

The pooling layer is constructed as 
\begin{align}
    \begin{split}
        \prod_{\alpha_i \in \mathcal{A}_i} C^0(i_{A,l}^{op},\R^3)(U_{\alpha_i}^i) &\to \prod_{\alpha_{i+1} \in \mathcal{A}_{i+1}} C^0(i_{A,l}^{op},\R^{3})(U_{\alpha_{i+1}}^{i+1}) \\
        (f)_{\alpha_{i}} &\mapsto F((f)_{\alpha_i})
    \end{split}
\end{align}
for some predetermined (possibly non-linear) pooling function $F: \R^{3|\mathcal{A}_i|} \to \R^{3 |\mathcal{A}_{i+1}|}$.

The fully connected layer is given by
\begin{align}
    \begin{split}
        \prod_{\alpha_i \in \mathcal{A}_i} C^0(i_{A,l}^{op},\R^3)(U_{\alpha_i}^i) \to C^0(\mathbb{R}^{3N^2}, \mathbb{R}^{3|\mathcal{A}_m|}) &\to C^0(i_{A,l}^{op},\R^{k})(I^2) \\
        (f)_{\alpha_{i}} \to (f)_{\alpha_i} &\mapsto F((f)_{\alpha_i})
    \end{split}
\end{align}
where $F: \R^{3|\mathcal{A}_m|} \to \R^k$ is a composition of trainable affine transformations and non-linear activation functions.
\end{example}

\begin{remark}
We note that vanilla convolutional neural network CNN is a discrete deep learning technique comprised of neighborhood aggregating layers. In addition, any convolutional layer and fully connected layer factor through inclusion (see Definition \ref{definition:fact_inc}), whereas maximum pooling layer does not.
\end{remark}

The example above immediately shows that convolutional neural networks (or image classifiers) comprised of convolutional layers, maximum pooling layers, and fully connected layers are subject to non-unique gluings of local features, adversarial attacks, and dependencies with respect to image datasets.
\begin{corollary}
Any convolutional neural network (or image data analyzing discrete deep learning technique) comprised of convolutional layers, maximum pooling layers, and fully connected layers satisfy Theorems \ref{theorem:first_limitation}, \ref{theorem:second_limitation}, and \ref{theorem:third_limitation}.
\end{corollary}

\begin{example}[Capsule Neural Networks]
Theorem \ref{theorem:first_limitation} verifies the motivation of capsule neural networks \cite{SFH17} that maximum pooling layers of conventional convolutional neural networks fail to observe how local components of images contribute to properties of global images. Instead of including pooling layers, the encoder of capsule neural networks are comprised of convolutional layers detecting geometric features (rather than RGB signals) of local components of images:
\begin{align}
    \begin{split}
        \prod_{\alpha_i \in \mathcal{A}_i} C^0(i_{A,l}^{op},\R^3)(U_{\alpha_i}^i) &\to \prod_{\alpha_{i+1} \in \mathcal{A}_{i+1}} C^0(i_{A,l}^{op},\R^{Ck})(U_{\alpha_{i+1}}^{i+1}) \to C^0(i_{A,l}^{op}, \R^{mk})(X)
    \end{split}
\end{align}
Here, the value $C$ corresponds to the number of capsules used to dissect geometric properties of local components of images, and the value $m$ corresponds to the number of classes the image dataset can be classified into. The layers provided above are trained from dynamic routings between $C$ primary capsules.
\end{example}

\begin{example}[Adversarial Attack]
Using the proposed framework, the effects of adversarial attack on the performance of convolutional neural networks can be easily analyzed \cite{GSS15}. Let $DL: \R^{3N^2} \to \R^k$ be the global section of $C^{0}(i_{A,l},\R^k)$ induced from the image of products of identity functions $(id, \cdots, id)$. Let $\widetilde{DL}: \R^{3N^2} \to \R^k$ be the global section $C^{0}(i_{A,l},\R^k)$ induced from the image of products of functions of form $(id + \nu_{m,n})_{m,n=1}^N$ where $\nu_{m,n}: \R^3 \to \R^3$ is a continuous function. The adversarial attack suggests that given any arbitrary input image $v \in \R^{3N^2}$, there exist constants $\epsilon, \delta > 0$ such that  
\begin{equation}
    |DL(v) - \widetilde{DL}(v)| > \epsilon
\end{equation}
even though $|\nu_{m,n}(x)| \leq \delta$ for every $x \in \R^3$ and $0 \leq m,n \leq N-1$.

Empirical evidences support that capsule neural networks are vulnerable to comparable forms of adversarial attacks imposed on classical convolutional neural networks \cite{MUUH19}. Because both architectures use neighborhood aggregating neural layers which factor through inclusion, Theorem \ref{theorem:second_limitation} implies that these techniques are subject to comparable adversarial attacks.
\end{example}

\begin{example} [Autoencoders]
The functorial formulation of autoencoders for image datasets can be defined as the following equation:
\begin{equation}
    \prod_{\alpha_0 = 1}^N C^0(i_{A,l}^{op}, \mathbb{R}^{3})(U^0_{\alpha_0}) \to C^0 (i_{A,l}^{op}, \mathbb{R}^k)(X) \to \prod_{\alpha_0 = 1}^N C^0(i_{A,l}^{op}, \mathbb{R}^{3})(U^0_{\alpha_0})
\end{equation}
where the first component of the composite function corresponds to compositions of layers used in convolutional neural networks, as shown in Definition \ref{example:cnn}, and the second component reconstructs the original image data by utilizing compositions of fully connected layers.
\end{example}

\subsection{Graph convolutional networks and Weisfeiler-Lehman kernels} \label{sec:GCN}

Before we formulate graph convolutional networks and Weisfeiler-Lehman kernels as global sections of a presheaf $C^{0}(i_{A,l},\R^k)$, we recall the theory of cell complexes and covering spaces, which can be utilized for analyzing the local topological properties of finite graphs. Readers who are interested in a rigorous treatment of this topic may refer to \cite{Ha02} or \cite{KV15}. 

We may consider a graph $G := (V,E)$ as a 1-dimensional CW complex, as constructed in Chapter 0 of \cite{Ha02}, where the nodes correspond to $0$-dimensional cells, and the edges correpond to $1$-dimensional cells. The open subsets of graphs consist of discrete sets of nodes, disjoint unions of open intervals defined over edges, and disjoint unions of open subsets rooted at a node $v$. Given a connected graph $G$, we denote by $\widetilde{G}$ its universal cover. We first construct a directed graph associated to an undirected graph without self-loops, as constructed in \cite{CPWC22}. 
\begin{definition}
\label{def:new_graph}
Given an undirected graph $G := (V,E)$ without self-loops, let $G' = (V,E')$ be a graph where every node has a self-loop of weight $1$. Let $G'' := (V,E'')$ be a directed graph where there exists a degree 2 projection map $p: G'' \to G'$ which is ramified at the set of nodes of $G'$. Denote by $\widetilde{G}''$ the universal covering space of $G''$.
\end{definition}
We now reformulate a well-known statement using sheaf theory that graph embedding frameworks over $G$ obtained from utilizing only the node labels are elements of the sections of $C^{0}(i_{A,l},\R^k)$ over $G$ or the universal cover of $G$, as carefully explored in \cite{KV15, XH19, B22, CPWC22}.
\begin{example} [Node-label based message passing neural networks]
Let $T$ be any graph embedding framework over a graph $G$ which is obtained from message passing neural networks utilizing only the node labels of $G$. Let $\widetilde{G}''$ be the universal covering space of $G''$, as constructed in Definition \ref{def:new_graph}. Let $U^k_x$ be the depth $k$ unfolding tree at the point $x \in \widetilde{G}''$. Fix a set of nodes $\widetilde{V}(G)$ consisting of lifts of nodes $v \in V(G)$ over $\widetilde{G}''$. Denote by $U^k \subset \widetilde{G}''$ the disjoint union of depth $k$ unfolding tree at each node in $\widetilde{V}(G)$.
\begin{equation}
    U^k := \bigsqcup_{v \in \widetilde{V}(G)} U^k_v
\end{equation}
Let $A_k subset \widetilde{G}''$ be the set of nodes given by
\begin{equation}
    A_k := U^{k} \cap V(\widetilde{G}'').
\end{equation}
Then there exists an integer $k_T > 0$ such that $T$ is the global section of the presheaf of continuous functions induced from direct sums of skyscraper cosheaves supported over $U^{k_T} \cap V(\widetilde{G}'')$.
\begin{equation}
    T \in C^{0} \left( i_{A_{k_T},l}^{op}, \mathbb{R}^{l'} \right)(\widetilde{G}'').
\end{equation}
\end{example}
We hence obtain that graph neural networks that are discrete deep learnign techniques with neighborhood aggregating layers possess limitations of the same kind to convolutional neural networks:
\begin{corollary}
Any graph neural network, embedding technique, or pooling methods utilizing only the node labels of a graph $G$ satisfy Theorems \ref{theorem:first_limitation}, \ref{theorem:second_limitation}, and \ref{theorem:third_limitation}.
\end{corollary}

\begin{example}[WL Tests and GCNs]
It is a well-known fact that Weisfeiler-Lehman isomorphism tests can distinguish two non-isomorphic graphs $G, G'$ as long as their universal covers $\widetilde{G}, \widetilde{G'}$ are not isomorphic \cite{SS11}. Note that $k$ iterations of Weisfeiler-Lehman isomorphism tests are comprised of $k$ layers which factor through inclusion.

In fact, the representations obtained from Weisfeiler-Lehman isomorphism test are global sections of the sheaf $\text{Hom}(i_{A_{k_T},l}^{op}, \mathbb{R}^l)$, because the neighborhood aggregating layers used in the Weisfeiler-Lehman procedure are linear. The histogram of node labels obtained from the Weisfeiler-Lehman procedure is obtained from the set of local sections on $U^k_v$, the depth $k$ unfolding trees at $v \in V(\widetilde{G}'')$. Hence, Theorem \ref{theorem:first_limitation} suggests that the Weisfeiler-Lehman isomorphism test can distinguish two non-isomorphic graphs with non-isomorphic universal covers. On the other hand, the flasqueness of the sheaf $\text{Hom}(i_{A,l}^{op},\R^l)$ implies that Weisfeiler-Lehman isomorphism test fails in distinguishing two non-isomorphic graphs with isomorphic universal covers, as demonstrated in Figure \ref{fig:sheafification_fail} and carefully studied in \cite{KV15, B22, CPWC22}.

The representations obtained from graph convolutional networks or message passing neural networks \cite{KW17}, on the other hand, correspond to global sections of the presheaf $C^0 (i_{A_{k_T},l}^{op}, \mathbb{R}^{l'})$, thanks to the composition of non-linear activation functions. Nevertheless, similar to Weisfeiler-Lehman isomorphism tests, these representations are also obtained from the set of local sections on $U^k_v$, the depth $k$ unfolding trees at $v \in V(\widetilde{G}'')$. 
\end{example}

\begin{figure}
                \centering
                \includegraphics[width=135mm]{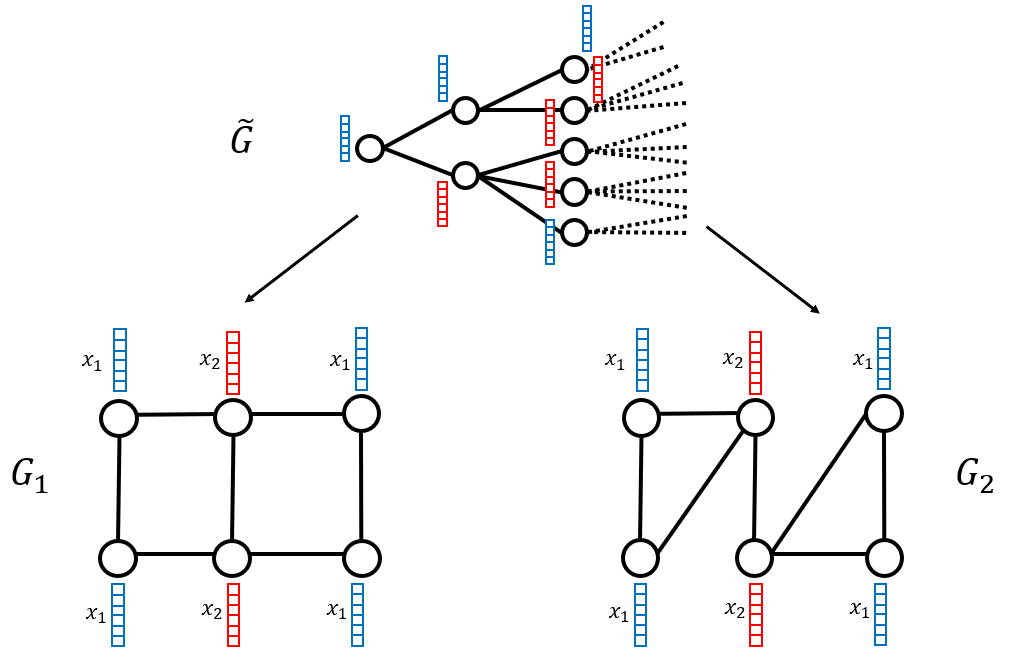}
                \caption{Sheafification: There are no obstructions imposed from topological characteristics of $X$ in obtaining global sections of $C^{0,+}(i_{A,l},\R^k)$. As such, non-trivial geometric differences among topological spaces cannot be effectively detected from neighborhood aggregating discrete deep learning techniques.}
                \label{fig:sheafification_fail}
\end{figure}

\begin{example}[Adversarial Attacks]
Similar to convolutional neural networks, empirical evidences suggest that graph neural networks also suffer from adversarial attacks arising from small graph perturbations or node feature perturbations \cite{ZG19}. This is not a surprising fact, because any graph neural network which includes a neighborhood aggregating layer factoring through inclusion can suffer from adversarial attacks, as suggested from Theorem \ref{theorem:second_limitation}.
\end{example}

\subsection{Recurrent Neural Networks} \label{sec:RNN}

One overarching insight obtainable from defining message passing neural networks as a global section of a presheaf is that the underlying topological space where the presheaf is defined is the universal cover of a graph $G$, considered as a 1-dimensional CW complex. This intuition can be generalized to devising a discrete deep learning technique which analyzes time series data over a given topological space $X$.

\begin{definition} [Time-Series data]
Given a topological space $X$, let $X_{S^1}$ be a principal $S^1$-bundle of $X$.
\begin{equation}
    S^1 \to X_{S^1} \to X
\end{equation}
The discrete deep learning technique $DL$ analyzing dynamic time series data over $X$ is a presheaf of continuous functions induced from skyscraper cosheaves defined over the universal cover of $X_{S^1}$, i.e. it is a function given as
\begin{equation}
    DL: \prod_{\alpha_0=1}^N C^0(i_{A,l}, \mathbb{R}^{k_0})(U_{\alpha_0}^0) \to C^0(i_{A,l}, \mathbb{R}^k)(\widetilde{X_{S^1}})
\end{equation}
comprised of compositions of layers
\begin{equation}
    \psi_n: \prod_{\alpha_{n-1} \in \mathcal{A}_{n-1}} C^0(i_{A,l}, \mathbb{R}^{k_{n-1}})(U_{\alpha_{n-1}}^{n-1}) \to \prod_{\alpha_n \in \mathcal{A}_n} C^0(i_{A,l}, \mathbb{R}^{k_n})(U_{\alpha_n}^n)
\end{equation}
with the associated sequence of collections of finitely many open subsets of $\widetilde{X_{S^1}}$.
\begin{equation}
    \left\{ \{U_{\alpha_0}^0\}_{\alpha_0=1}^N, \{U_{\alpha_1}^1\}_{\alpha_1 \in \mathcal{A}_1}, \cdots, \{U_{\alpha_m}^m\}_{\alpha_m \in \mathcal{A}_m}, \widetilde{X_{S^1}} \right\}
\end{equation}
\end{definition}
Namely, we associate the time variable $t$ to a fiber of the covering map $\pi: \widetilde{X_{S^1}} \to X_{S^1} \to X$ at a point $x \in X$ homeomorphic to $S^1$. 

\begin{example} [Recurrent Neural Network]
The real line $\mathbb{R}$ is the universal cover of the circle $S^1 \times \{x\}$, which can be considered as a principal $S^1$-bundle of a singleton space $\{x\}$. 

From this perspective, the discrete deep learning techniques specialized for analyzing time series data (such as recurrent neural networks $RNN$ or long short-term memory networks $LSTM$) is a global section of a presheaf of continuous functions induced from direct sums of skyscraper cosheaves over $\widetilde{S^1}$.
\begin{equation}
    DL: \prod_{\alpha_0=1}^N C^0(i_{A,l}, \mathbb{R}^{k_0})(U_{\alpha_0}^0) \to C^0(i_{A,l}, \mathbb{R}^k)(\widetilde{S^1})
\end{equation}
The set of points $\chi_\mathbb{R} := \{x_1, x_2, \cdots, x_N\}$ lie in the preimage of the base point $x \in S^1$ with respect to the covering map $\pi: \mathbb{R} \to S^1$. For recurrent neural networks $RNN$, the elements of the finite collection of open subsets $\{U_{\alpha_i}^i\}_{\alpha_i \in \mathcal{A}_i}$ consist of connected open subsets of $\mathbb{R}$ such that there exists an integer $1 \leq m \leq N$ and an index $\alpha_i \in \mathcal{A}_i$ such that
\begin{equation}
    U_{\alpha_i}^i \supset \{x_1, x_2, \cdots, x_m\}
\end{equation}
As for long short-term memory networks $LSTM$ \cite{HS97} or gated recurrent units $GRU$ \cite{CM14}, the elements of the finite collection of open subsets $\{U_{\alpha_i}^i\}_{\alpha_i \in \mathcal{A}_i}$ consist of connected open subsets of $\mathbb{R}$ such that there exists an integer $1 \leq m_1 < m_2 \leq N$ and an index $\alpha_i \in \mathcal{A}_i$ such that
\begin{equation}
    U_{\alpha_i}^i \supset \{x_{m_1}, x_{m_1+1}, \cdots, x_{m_2} \}
\end{equation}
represented by their memory cells.
\end{example}

\begin{example} [Positional Enconding in attention transformers]
The positional encoding used in the encoder layer of the attention-transformer implicitly uses the assumption that the time variable $t$ can be identified with the topological space $S^1$. Indeed, the position of an input data is defined in terms of trigonometric functions $f: \mathbb{R} \to S^1$.
\end{example}

Because vanilla RNNs, and LSTMs are neighborhood aggregating discrete deep learning techniques, we immediately obtain the following result.
\begin{corollary}
Any recurrent neural networks or long short term memory networks (LSTM) satisfy Theorems \ref{theorem:first_limitation}, \ref{theorem:second_limitation}, and \ref{theorem:third_limitation}.
\end{corollary}

\begin{example}
Empirical evidences suggest that adversarial attacks on recurrent neural networks and long short term memory networks can be achieved from small perturbations of input sequential data \cite{ZS20}, \cite{PD16}, which is precisely what the proof of Theorem \ref{theorem:second_limitation} suggests.
\end{example}

\section{Beyond Neighborhood Aggregating Discrete Deep Learning Techniques} \label{section:deep_learning_future}

In the previous sections, we observed how the fact that the presheaf $C^0(i_{A,l}, \mathbb{R}^k)$ is not a sheaf gave rise to limitations of discrete deep learning techniques. We end the paper with a formulation of deep learning techniques which do not necessarily fit in the class of discrete deep learning techniques with neighborhood aggregating layers, and propose future research directions on what novel deep learning algorithms may focus on.

\subsection{Non-neighborhood Aggregating Layers} 

As observed from Theorems \ref{theorem:second_limitation} and \ref{theorem:third_limitation}, the four neighborhood aggregating axioms make discrete deep learning techniques vulnerable to adversarial attacks, and prevents them from obtaining all possible vector representations of any arbitrarily given input datasets. Hence, it is natural to consider whether it is possible to devise a discrete deep learning algorithm whose layers do not satisfy the neighborhood aggregating axioms. Attention-transformers as constructed in Vaswani et al \cite{VS17} do not satisfy the neighborhood aggregating axioms.

\begin{example}[Attention-transformers]
Let $X := S^1 \times [0,1]$ be a cylinder. Fix a positive integer $N, d, w > 0$.
\begin{itemize}
    \item \textbf{Positional Encoding:} Denote by $A := \{x_{i,j}\}_{\substack{1 \leq i \leq N \\ 1 \leq j \leq d}}$ the set of points on $X$ whose locations are determined as
    \begin{equation}
        x_{i,j} = \begin{cases}
        \left( \sin\left( \frac{i}{10000^{\frac{2j}{d}}} \right), \frac{2j-1}{2d} \right) &\text{ if } i = 2k \\
        \left( \cos\left( \frac{i}{10000^{\frac{2j}{d}}} \right), \frac{2j-1}{2d} \right) &\text{ if } i = 2k+1
        \end{cases}
    \end{equation}
    \item \textbf{Skyscraper Cosheaf:} The pushforward cosheaf $i_{A,l}^{op}$ is given by the direct sum of skyscraper sheaves of real vector spaces of dimension $1$ supported at $x_{i,j} \in A$. Note that any set of $d$-dimensional $N$ vectors $\{v_1,\cdots,v_N\}$ induces a global section of $i_{A,l}^{op}$.
    \item \textbf{Open Cover:} Let $\{U_{\alpha_0}\}_{\alpha_0=1}^{Nd}$ be the set of finite open covers of $X$ such that for each $\alpha_0$ there exists a unique $1 \leq i \leq N$ and $1 \leq j \leq d$ such that $A \cap U_{\alpha_0} = \{x_{i,j}\}$. Let $\{U_{\alpha_1}\}_{\alpha_1 = 1}^N$ be the set of finite open covers of $X$ such that for each $\alpha_1$ there exists a unique $1 \leq i \leq N$ such that $A \cap U_{\alpha_1} = \{x_{i,j}\}_{j=1}^d$.
    \item \textbf{Attention-Transformer:} The attention transformer $TR$ \cite{VS17} can be identified as
    \begin{equation}
        TR: \prod_{\alpha_0=1}^{Nd} C^0(i_{A,l}^{op},\mathbb{R})(U_{\alpha_0}) \to C^0(i_{A,l}^{op},\mathbb{R}^{Nd})(X)
    \end{equation}
    comprised of compositions of encoder and decoder layers. Both are compositions of multi-head attention functions and feed forward networks.
    \item \textbf{Multi-head Attention:} Denote by $h$ the number of attention heads. Denote by $Q,K,V$ the querry, key, and the value matrix obtained from either the input data or the output of the encoder layer. Denote by $\{W_Q^i, W_K^i, W_V^i\}_{i=1}^h$ the collection of $N \times w$  weight matrices for all attention heads. Denote by $W_Z$ a $wh \times d$ weight matrix. A single multi-head attention layer is defined as the composition of feedforward networks and multihead attention functions.
    \begin{align}
    \begin{split}
        W_V: \prod_{\alpha_0}^{Nd} C^0(i_{A,l}^{op},\mathbb{R})(U_{\alpha_0}) &\to \prod_{\alpha_1=1}^N C^0(i_{A,l}^{op}, \mathbb{R}^{wh})(U_{\alpha_1}) \\
        V &\mapsto \left(VW_V^i\right)_{i=1}^h \\
        \text{MultiHead}: \prod_{\alpha_1=1}^N C^0(i_{A,l}^{op}, \mathbb{R}^{wh})(U_{\alpha_1}) &\to C^0(i_{A,l}^{op},\mathbb{R}^{Nwh})(X) \\
        (V_i)_{i=1}^h &\mapsto \left(\text{Softmax}\left( \frac{Q W_Q^i (K W_K^i)^T}{\sqrt{d_k}} \right) V_i \right)_{i=1}^h\\
        W_Z: C^0(i_{A,l}^{op},\mathbb{R}^{Nwh})(X) &\to C^0(i_{A,l}^{op},\mathbb{R}^{Nd})(X) \\
        (Z_i)_{i=1}^h &\mapsto W_Z (Z_i)_{i=1}^h
    \end{split}
    \end{align}
\end{itemize}
\end{example}

\begin{remark}
We note that a single encoder layer of the attention-transformer is equivalent to the formulation of dynamic routing of capsules from capsule neural networks. In particular, the two neural networks are equivalent if the number of capsules $C$ is equal to the number of multi-head attentions $h$. 
\end{remark}

\begin{remark}
The positional encoding function of the attention-transformer projects the input data set $D$ defined over the topological space $Y$ to a cylinder $S^1 \times [0,1]$. Because the attention-transformer is a discrete deep learning technique, it is unable to fully encapsulate the topological properties of $Y$. This is not an issue as long as the underlying topological space $Y$ is contractible or homotopic to $S^1$. For example, natural languages or images are data defined over the euclidean space $\mathbb{R}^d$, which is contractible. Variants of transformers are experimentally shown to produce state of the art results in natural language processes \cite{VS17} and image classifications \cite{MMD21}. In fact, the projection of a 2-dimensional Euclidean space to a cylinder is far from a preposterous construction, as the space $\mathbb{R}$ is a universal cover of $S^1$, which also induces a covering map $\mathbb{R}^2 \to S^1 \times \mathbb{R}$.
\end{remark}

\begin{remark}
Attention-transformer with more than two encoder and decoder layers is an example of a non-neighborhood aggregating discrete deep learning technique. The first encoder and decoder layer of the attention-transformer is a neighborhood aggregating layer, whereas the other layers do not satisfy (1) and (4) of the neighborhood aggregating axioms (from Definition \ref{def:neighborhood_aggregate}). Hence, we observe that Theorems \ref{theorem:first_limitation} and \ref{theorem:second_limitation} still remain valid for attention-transformers, whereas conditions for Theorem \ref{theorem:third_limitation} are not satisfied. To elaborate, Theorem \ref{theorem:first_limitation} provides a theoretical explanation for hallucinations observed in many transformer architecture, as explored in \cite{HZ25, SS25}. Here, we may reinterpret hallucinations as limitations in gluing short length sentences or words coherently to form a lengthy coherent logical statement. Theorem \ref{theorem:second_limitation} gives a mathematical formulation of vulnerability of transformer architecture against adversarial attacks, as pointed out empirically in recent studies \cite{FG25, SK26}. Nevertheless, as the conditions for Theorem \ref{theorem:third_limitation} are not satisfied, one may expect that transformers could have capabilities to outperform other conventional message passing neural networks. Indeed, empirical evidences support that performances of attention transformers in processing both natural language processes \cite{HC19} and image classifications \cite{DBKWZ21} outperform deep learning techniques which utilize classical convolutional neural networks or residual neural networks.

One interesting property of attention-transformers is that under certain rigid conditions the attention weight matrix obtained from the multihead attention function can be approximated by sparse matrices. Let $P \in (0,1)$, and $\{U_\alpha\}_{\alpha \in \mathcal{A}}$ be a finite open cover of a cylinder $X = S^1 \times [0,1]$ such that $|\mathcal{A}| = N$. One may ask a question whether for any input data of dimension $Nd$ there exists a choice of predetermined $PNd$ values, a finite open cover $\{V_j\}_{j=1}^{h+1}$ of $X$ which satisfy
\begin{equation}
        V_j = 
        \begin{cases}
        \cap_{\alpha' \in \mathcal{A}_j' \subsetneq \mathcal{A}}U_{\alpha'} & \text{ if } 1 \leq j \leq h \\
        X \setminus \cup_{j=1}^h V_j & \text{ if } j = h+1,
        \end{cases}
\end{equation}
and a fixed constant $c > 0$ such that for each $1 \leq j \leq h$, the open subsets $V_j's$ satisfy
\begin{equation}
    \# \left( V_j \cap A \right) < PNd + c.
\end{equation}
Suppose further that for each $1 \leq j \leq h$, there exists a morphism
\begin{equation}
    \prod_{\alpha \in \mathcal{A}} C^0(i_{V_j \cap A,l}^{op}, \mathbb{R}^{wh})(U_{\alpha}) \to \prod_{m = 1}^h C^0(i_{V_j \cap A,l}^{op}, \mathbb{R}^{Nw})(V_m)
\end{equation}
such that the following commutative diagram holds:
\begin{center}
    \begin{tikzcd}
    \prod_{\alpha \in \mathcal{A}} C^0(i_{V_j \cap A,l}^{op}, \mathbb{R}^{wh})(U_{\alpha}) \arrow[r, dotted, "\exists"] \arrow[dr, yshift=-5pt, "\text{MultiHead}_{U_\alpha}"] & \prod_{m = 1}^h C^0(i_{V_j \cap A,l}^{op}, \mathbb{R}^{Nw})(V_m) \arrow[d, "\text{MultiHead}_{V_m}"] \\
    \prod_{\alpha \in \mathcal{A}} C^0(i_{A,l}^{op}, \mathbb{R}^{wh})(U_{\alpha}) \arrow[u] & C^0(i_{A,l}^{op}, \mathbb{R}^{Nwh})(X)
    \end{tikzcd}
\end{center}
Then such a collection of finite open sets $\{V_j\}_{j=1}^h$ always exist for any value of $P \in (0,1)$ if the probability distribution of the data $\{x_i\}_{i=1}^N$ of sufficiently large enough dimension $d$ is a Gaussian distribution with mean $0$, and the entries of the random weight matrices $W_Q$, $W_K$, and $W_V$ also form a Gaussian distribution with mean $0$. The data points $\{x_i\}$ project to form a uniform distribution over the sphere $S^d$, and it is a classical result that for any matrix $W$, the probability distribution of $x_i W x_i^T$ for a unit vector $x_i$ converges to the Gaussian distribution with mean $\frac{1}{d} \text{Tr}(W)$, and the probability distribution of two i.i.d. unit vectors $\langle x_i, x_j \rangle$ is the beta distribution $(\frac{d-1}{2}, \frac{d-1}{2})$. By fixing the data inputs with sufficiently large enough high attention scores, we can force the attention scores obtained from other data inputs except for possibly at $c$ inputs to be sufficiently close to $0$. 

Furthermore, Hahn proves that such a collection $\{V_j\}_{j=1}^h$ always exist for input data consisting of 1-dimensional binary entries, where the predetermined values are chosen from components with high attention scores. This implies that attention-transformers with bounded number of layers or heads is not capable of modeling periodic or hierarchical structure of finite-state languages \cite{Ha20}. 

Nevertheless, the same limitation can also make attention-transformers robust to perturbations in input data, because it constrains the weighted sum of coordinate-wise perturbations to from growing arbitrarily large, assuming that the value $P$ is sufficiently close to $0$. Indeed, as demonstrated in \cite{HC19} and \cite{Ha20} under rigid constraints on the input data, attention-transformers are observed to be more robust to small perturbations of input data compared to other neighborhood aggregating discrete deep learning techniques.
\end{remark}

\begin{remark}
Note that the encoder and decoder layers other than the first ones admit the output of the previous encoder and decoder layer as inputs. This construction in fact redefines the layers as a function from $C^0(i_{A,l}^{op}, \mathbb{R}^{Nd})(X)$ to itself.
\begin{equation}
    \prod_{\alpha_i}^{Nd} C^0(i_{A,l}^{op}, \mathbb{R})(X) \to C^0(  i_{A,l}^{op}, \mathbb{R})(X)
\end{equation}
This gives us a hint as to why attention-transformers are able to outperform other neighborhood aggregating discrete deep learning techniques in analyzing natural languages and images, because the composition of encoder and decoder layers allows one to construct a wider range of continuous functions. Required for a thorough assessment on the strengths of encoder and decoder layers is a careful analysis on how the image of the attention-transformer evolves with respect to the increase in the number of encoder and decoder layers.
\end{remark}

\subsection{Sheaves other than skyscraper cosheaves}

As aforementioned, the fact that $C^0(i_{A,l}, \mathbb{R}^k)$ does not satisfy the sheaf axioms gives rise to limitations of discrete deep learning techniques. Therefore, future research may focus on constructing a deep learning technique associated to a sheaf $\mathcal{F}$ other than the presheaf of continuous functions induced from skyscraper cosheaves.

\begin{definition}[Deep Learning Technique associated to a sheaf / cosheaf]
Let $X$ be a locally compact topological space. Let $\mathcal{F}$ be a sheaf (or a cosheaf) of real vector spaces over $X$. Consider a sequence of collections of finite open subsets 
\begin{equation}
    \left\{ \{U_{\alpha_0}^0\}_{\alpha_0=1}^N, \{U_{\alpha_1}^1\}_{\alpha_1 \in \mathcal{A}_1}, \cdots, \{U_{\alpha_m}^m\}_{\alpha_m \in \mathcal{A}_m}, X \right\}
\end{equation}

A deep learning algorithm associated to a sheaf (or a cosheaf) $\mathcal{F}$ with $m$ layers, denoted as $DL_\mathcal{F}^m$, is a well-defined composition of functions given as:
\begin{equation}
     \prod_{\alpha_0=1}^N \mathcal{F}(U_{\alpha_0}) \to \prod_{\alpha_1 \in \mathcal{A}_1} \mathcal{F}(U_{\alpha_1}) \to \cdots \to \prod_{\alpha_m \in \mathcal{A}_m} \mathcal{F}(U_{\alpha_m}) \to \mathcal{F}(X)
\end{equation}

Each function, possibly non-linear,
\begin{equation}
    \psi_{i+1}: \prod_{\alpha_i \in \mathcal{A}_i} \mathcal{F}(U_{\alpha_i}) \to \prod_{\alpha_{i+1} \in \mathcal{A}_{i+1}} \mathcal{F}(U_{\alpha_{i+1}})
\end{equation}
corresponds to the $i+1$-th layer of the deep learning technique associated to $\mathcal{F}$.
\end{definition}

\begin{remark}[]
Any graph neural networks which utilizes both node labels and the gluing information of subgraphs of a graph $G$ are global sections of the presheaf $C^0(i_{A,l}, \R^k)(G)$, which forgets non-trivial topological invariants of the input graph $G$. This is because the skyscraper cosheaf is flasque, forgetting any gluing structure.

Examples of graph neural networks incorporating richer data than skyscraper cosheaves include persistent homological techniques \cite{CC20, RBB19} over graphs, cellular sheaves \cite{BB22, BB22b, BBO24, BG22} over graphs introduced from \cite{Cu13}, and copresheaves over combinatorial complexes \cite{HB25}. Future research may focus on whether using different sheaves over a topological space allows one to evade limitations of conventional discrete deep learning techniques suggested from Theorems \ref{theorem:first_limitation}, \ref{theorem:second_limitation}, and \ref{theorem:third_limitation}.
\end{remark}

\begin{remark}
Neural ordinary differential equations can be considered as a deep learning technique associated to the sheaf of differential equations $\Omega_X$ \cite{CRBD18}. Compared to recurrent neural networks, neural ODEs are more effective than recurrent neural networks in detecting temporal patterns among time series data which are governed by an underlying globally defined differential equation. Such enhanced performances may originate from the fact that $\Omega_X$ is a sheaf over a topological space, which guarantees that the restriction map of global sections to products of local sections is injective, whereas recurrent neural networks are modeled by sheaves $C^0(i_{A,l},\R^k)$ which fails the respective condition on the restriction map (see Theorem \ref{theorem:first_limitation}). 
\end{remark}

\subsection{Principal circle bundles and dynamic graphs}
We conclude the discussion on formulating a functorial model for neighborhood aggregating deep learning methods with a demonstration on the correspondence between discrete deep learning techniques which process time series data over graphs (considered as a 1-dimensional CW complex) and those which process data over 2-dimensional topological spaces, such as image data sets. These correspondences, as empirically demonstrated from recent breakthroughs on utilizing attention transformers to image classification tasks \cite{DBKWZ21}, can be obtained from the observation that universal covers of 2-dimensional orientable smooth connected manifolds with non-trivial genus are homeomorphic to either $\mathbb{R}^2$ or the 2-dimensional hyperbolic disk $D^2$. To rigorously formulate these observations, we first recall the definition of torsors, which we closely follow the exposition from Chapters 4,5, and 6 of \cite{St21}. 
\begin{definition}[Torsors]
Let $\mathcal{F}$ be a sheaf of abelian groups over $X$. A $\mathcal{F}$-torsor is a sheaf of sets $\mathcal{G}$ on $X$ with an action $\mathcal{F} \times \mathcal{G} \to \mathcal{G}$ such that
\begin{enumerate}
    \item For any open neighborhood $U \subset X$, the action $\mathcal{F}(U) \times \mathcal{G}(U) \to \mathcal{G}(U)$ is simply transitive.
    \item For every $x \in X$ and every open neighborhood $U $ containing $x$, the set $\mathcal{G}(U)$ is nonempty (i.e. the stalk $\mathcal{G}_x$ is nonempty).
\end{enumerate}
\end{definition}

\begin{example}
The sheaf $\mathcal{F}$ is the trivial $\mathcal{F}$-torsor, endowed with the action induced from left multiplication.
\begin{align}
\begin{split}
    \mathcal{F}(U) \times \mathcal{F}(U) &\to \mathcal{F}(U) \\
    (f, g) &\mapsto fg
\end{split}
\end{align}
\end{example}

\begin{example}
Let $GL_{n,X}$ be the constant sheaf of general linear group $GL_n$ over the real manifold $X$. The $GL_{n,X}$-torsors correspond to rank $n$ vector bundles $f: Y \to X$ over $X$.
\end{example}

\begin{example}
Let $G$ be a group. Let $X$ be a topological space such that $G$ acts over $X$ endowed with the action $X \times G \to X$ that satisfies
\begin{enumerate}
    \item $x \cdot 1 = x$
    \item $x \cdot (gh) = (x \cdot g) \cdot h$
    \item The map $G \to X$ given by $g \mapsto x \cdot g$ is bijective.
\end{enumerate}
\begin{align}
    \begin{split}
        x \cdot 1 &= x \\
        x \cdot (gh) &= (x \cdot g) \cdot h
    \end{split}
\end{align}
By definition, $X$ is a $G$-torsor, where $G$ is considered as a constant sheaf over $X$. These topological spaces are also known as $G$-principal homogeneous spaces. For instance, there exists a natural action of the special orthogonal group $SO(2)$ consisting of rotational symmetries over the complex circle group $S^1_{\mathbb{C}}$.
\end{example} 

There is a canonical relation between the first sheaf cohomology group of $\mathcal{F}$ and the set of isomorphism classes of $\mathcal{F}$-torsors over $X$.
\begin{lemma}
There exists a canonical bijection between $H^1(X,\mathcal{F})$ and the set of isomorphism classes of $\mathcal{F}$-torsors.
\end{lemma}
\begin{proof}
We refer to Lemma 4.3 of \cite{St21}.
\end{proof}

Using the lemma provided above, it is a classical result that isomorphism classes of principal $S^1$ bundles of a topological space $X$ are in bijection with second cohomology classes of $X$.
\begin{lemma} \label{lemma:sphere_bundle}
There exists a canonical bijection between $H^2(X,\mathbb{Z})$ and the set of isomorphism classes of principal $S^1$ bundles of $X$.
\end{lemma}
\begin{proof}
The result follows from using a long exact sequence associated to the exponential map $0 \to \Z \to \Z \to S^1 \to 0$.
\end{proof}

We hence obtain the following correspondence between classical deep learning techniques and deep learning techniques for representing dynamic graphs with non-trivial cycles.
\begin{theorem} \label{theorem:dynamic_graph}
\begin{enumerate}
    \item The MPNN processing a dynamic dataset over a graph $G$ homeomorphic to $S^1$, a deep learning technique processing sets of inputs in $\mathbb{R}^2$, or that processing sets of inputs in $S^1 \times \mathbb{R}$ endowed with the Euclidean metric are equivalent.
    \item The MPNN processing a dynamic dataset over a dynamic graph $G$ such that $\text{rank}_{\Z} H^1(G,\Z)$ is at least $2$ is equivalent to a deep learning technique processing sets of inputs in a subset $S$ of a closed 2-dimensional disk $D^2$ endowed with the hyperbolic metric, such that the interior of $S$ is the open 2-dimensional disk.
\end{enumerate}
\end{theorem}
\begin{proof}
By Lemma \ref{lemma:sphere_bundle}, every principal $S^1$ bundle of a 1-dimensional graph $G$ is homeomorphic to $G \times S^1$, which is a compact manifold with (or without) boundary. We first consider the case where $G \cong S^1$. Then there exist covering maps $\mathbb{R}^2 \to S^1 \times S^1$ and $S^1 \times \mathbb{R} \to S^1 \times S^1$, both of which are endowed with the Euclidean metric. If $\mathrm{rank}_{\mathbb{Z}} H^1(G, \mathbb{Z}) \geq 2$, then we proceed as in \cite{Kohan14}. The compact manifold $G \times S^1$ is a Riemann surface endowed with a hyperbolic metric with geodesic boundary. Its universal cover is a subset $S$ of a closed $2$-dimensional disk $D^2$ endowed with the hyperbolic metric constructed as follows: one has $\overline{S} = D^2$, and $D^2 \setminus S$ is a subset of $\partial D^2$ homeomorphic to an infinite disjoint union of open circular arcs and a cantor set, see for example Theorem 3.4.6 of \cite{Katok92} (Here we identify $G \times S^1$ as a quotient of $S$ by the action of a Fuchsian group of second kind, see page 67, item (b) of \cite{Katok92} for its definition). Let $\widetilde{G \times S^1}$ be a covering space of $G \times S^1$. Let $A$ be the set of finitely many discrete points over $\widetilde{G \times S^1}$. Then the theorem immediately follows from the fact that the representations obtained from such MPNN techniques are global sections of the presheaf $C^0(i_{A,l}^{op}, \mathbb{R}^k)(\widetilde{G \times S^1})$.
\end{proof}

\begin{remark}
Theorem \ref{theorem:dynamic_graph} verifies the empirical result that visual transformers with $1$-dimensional positional encodings can be utilized to classify 2-dimensional image datasets, and produce state of the art results in large datasets compared to other convolutional neural networks \cite{DBKWZ21}.
\end{remark}

\end{document}